\documentclass{article}
\usepackage{arxiv}
\usepackage{microtype}
\usepackage{hyphenat}
\makeatother

\usepackage{lmodern}

\usepackage{microtype}
\usepackage{hyperref}
\usepackage{url}
\usepackage{booktabs}
\usepackage{minted}
\usepackage{float} 
\usepackage{microtype}
\usepackage{caption}
\usepackage{subcaption}
\usepackage{hyperref}
\usepackage{url}
\usepackage{amsthm}
\usepackage{bm}
\usepackage{bbm}

\newtheorem{proposition}{Proposition}

\usepackage{booktabs}
\usepackage{textcomp}
\usepackage{mathabx}
\usepackage{mathabx}
\usepackage{bbm}

\definecolor{darkblue}{rgb}{0, 0, 0.5}
\hypersetup{colorlinks=true, citecolor=darkblue, linkcolor=darkblue, urlcolor=darkblue}

\RequirePackage[T1]{fontenc}
\RequirePackage[tt=false, type1=true]{libertine}
\RequirePackage[varqu]{zi4}

\usepackage{etoolbox}
\usepackage{appendix}
\usepackage{xspace}
\newcommand{\eg}{e.g.\xspace}
\newcommand{\ie}{i.e.\xspace}

\usepackage{amsmath, amssymb, amsthm}
\usepackage{algorithm}
\usepackage{algorithmic}
\usepackage{hyperref}
\usepackage{booktabs}
\usepackage{enumitem}

\usepackage{tikz}
\usetikzlibrary{shapes.geometric, arrows.meta, positioning, calc, fit, backgrounds, patterns,patterns.meta,matrix, decorations.pathreplacing, decorations.pathmorphing,calc, 
                  shadows.blur, positioning}

\usepackage[normalem]{ulem}
\definecolor{border}{HTML}{1E293B}
\definecolor{borderm}{HTML}{334155}
\definecolor{vis}{HTML}{3B82F6}
\definecolor{visDark}{HTML}{1D4ED8}
\definecolor{visLight}{HTML}{DBEAFE}
\definecolor{visBg}{HTML}{EFF6FF}
\definecolor{evictA}{HTML}{EF4444}
\definecolor{evictLight}{HTML}{FEE2E2}
\definecolor{accent}{HTML}{10B981}
\definecolor{accentDark}{HTML}{047857}
\definecolor{accentLight}{HTML}{D1FAE5}
\definecolor{traincolor}{HTML}{7C3AED}
\definecolor{trainLight}{HTML}{EDE9FE}
\definecolor{labelc}{HTML}{475569}
\definecolor{gridkept}{HTML}{93C5FD}
\definecolor{promptc}{HTML}{64748B}
\definecolor{tokbg}{HTML}{F8FAFC}

\usepackage{amsmath,amsfonts,bm}

\def\eqref#1{equation~\ref{#1}}

\def\1{\bm{1}}

\DeclareMathAlphabet{\mathsfit}{\encodingdefault}{\sfdefault}{m}{sl}
\SetMathAlphabet{\mathsfit}{bold}{\encodingdefault}{\sfdefault}{bx}{n}

\usepackage{lineno}

\definecolor{darkblue}{rgb}{0, 0, 0.5}
\hypersetup{colorlinks=true, citecolor=darkblue, linkcolor=darkblue, urlcolor=darkblue}
\usepackage{tikz}
\usetikzlibrary{arrows.meta, positioning, decorations.pathreplacing, calc, patterns}

\definecolor{border}{HTML}{3B4252}
\definecolor{borderm}{HTML}{4C566A}
\definecolor{labelc}{HTML}{6B7394}
\definecolor{tokbg}{HTML}{F0F1F5}
\definecolor{vis}{HTML}{5E81AC}
\definecolor{visDark}{HTML}{3B6EA5}
\definecolor{visLight}{HTML}{D8E6F3}
\definecolor{visBg}{HTML}{E8F0F8}
\definecolor{evictA}{HTML}{BF616A}
\definecolor{evictLight}{HTML}{F2DEDE}
\definecolor{gridkept}{HTML}{A3BE8C}
\definecolor{promptc}{HTML}{8B7EC8}
\definecolor{accent}{HTML}{4CAF50}
\definecolor{accentDark}{HTML}{2E7D32}
\definecolor{accentLight}{HTML}{E8F5E9}
\definecolor{traincolor}{HTML}{D08C3E}
\definecolor{trainLight}{HTML}{FDF0E0}
\definecolor{budgetcolor}{HTML}{C0392B}
\definecolor{budgetBg}{HTML}{FDF2F0}
\definecolor{border}{HTML}{4A4A4A}
\definecolor{gradcol}{HTML}{C0392B}
\definecolor{autogradcol}{HTML}{E67E22}

\definecolor{sumcolor}{HTML}{E67E22}
\definecolor{sumLight}{HTML}{FDF2E9}
\definecolor{sumDark}{HTML}{BF6516}
\definecolor{sumKept}{HTML}{F0C27A}

\definecolor{vis}{HTML}{4A90D9}       
\definecolor{evictA}{HTML}{E05A4F}    
\definecolor{evictB}{HTML}{E8A735}    
\definecolor{fut}{HTML}{F0F0F0}       
\definecolor{border}{HTML}{4A4A6A}
\definecolor{labelc}{HTML}{3A3A5C}

\definecolor{metaBg}{HTML}{FFFBF5}
\definecolor{metaBorder}{HTML}{C05621}
\definecolor{metaText}{HTML}{9C4221}
\definecolor{dumpColor}{HTML}{7B2D8E}
\definecolor{dumpLight}{HTML}{F5E6FF}

\title{Neural Garbage Collection:\\ Learning to Forget while Learning to Reason}
\author{\normalsize
Michael Y.~Li\thanks{Correspondence to \texttt{michaelyli@stanford.edu}}\\
\normalsize
Stanford University
\And
\normalsize
Jubayer Ibn Hamid\\
\normalsize
Stanford University
\And
\normalsize
Emily B.~Fox\\
\normalsize
Stanford University
\And
\normalsize
Noah D.~Goodman\\
\normalsize
Stanford University
}

\begin{document}

\maketitle
\begin{abstract}
Chain-of-thought reasoning has driven striking advances in language model capability, yet every reasoning step grows the KV cache, creating a bottleneck to scaling this paradigm further.
Current approaches manage these constraints on the model’s behalf using hand-designed criteria. 
A more scalable approach would let end-to-end learning subsume this design choice entirely, following a broader pattern in deep learning. 
After all, if a model can learn to reason, why can't it learn to forget?
We introduce \textbf{Neural Garbage Collection (NGC)}, in which a language model learns to forget while learning to reason, trained \emph{end-to-end} from outcome-based task reward alone. 
As the model reasons, it periodically pauses, decides which KV cache entries to evict, and continues to reason conditioned on the remaining cache.
By treating tokens in a chain-of-thought and cache-eviction decisions as discrete actions sampled from the language model, we can use reinforcement learning to jointly optimize how the model reasons and how it manages its own memory: what the model evicts shapes what it remembers, what it remembers shapes its reasoning, and the correctness of that reasoning determines its reward.
Crucially, the model learns this behavior entirely from a \emph{single learning signal} — the outcome-based task reward — without supervised fine-tuning or proxy objectives.
On Countdown, AMC, and AIME tasks, NGC maintains strong accuracy relative to the full-cache upper bound at 2–3x peak KV cache size compression and substantially outperforms eviction baselines.
Our results are a first step towards a broader vision where end-to-end optimization drives both capability and efficiency in language models.
\end{abstract}
\section{Introduction}
\definecolor{cAmberFill}  {HTML}{FEF3C7}  \definecolor{cAmberStr}  {HTML}{F59E0B}  \definecolor{cAmberTxt}  {HTML}{92400E}
\definecolor{cRoseFill}   {HTML}{FFE4E6}  \definecolor{cRoseStr}   {HTML}{FB7185}  \definecolor{cRoseTxt}   {HTML}{9F1239}
\definecolor{cSkyFill}    {HTML}{DBEAFE}  \definecolor{cSkyStr}    {HTML}{60A5FA}  \definecolor{cSkyTxt}    {HTML}{1E3A8A}
\definecolor{cSkyBg}      {HTML}{EFF6FF}
\definecolor{cVioletFill} {HTML}{EDE9FE}  \definecolor{cVioletStr} {HTML}{8B5CF6}  \definecolor{cVioletTxt} {HTML}{5B21B6}
\definecolor{cGreenFill}  {HTML}{DCFCE7}  \definecolor{cGreen}     {HTML}{22C55E}  \definecolor{cGreenTxt}  {HTML}{166534}
 
\definecolor{tC0Fill}{HTML}{FECACA}  \definecolor{tC0Str}{HTML}{B91C1C}  \definecolor{tC0Txt}{HTML}{7F1D1D}
\definecolor{tC1Fill}{HTML}{FEE2E2}  \definecolor{tC1Str}{HTML}{EF4444}  \definecolor{tC1Txt}{HTML}{991B1B}
\definecolor{tC2Fill}{HTML}{FED7AA}  \definecolor{tC2Str}{HTML}{EA580C}  \definecolor{tC2Txt}{HTML}{7C2D12}
\definecolor{tC3Fill}{HTML}{FFEDD5}  \definecolor{tC3Str}{HTML}{F97316}  \definecolor{tC3Txt}{HTML}{9A3412}
\definecolor{tC4Fill}{HTML}{FEF3C7}  \definecolor{tC4Str}{HTML}{F59E0B}  \definecolor{tC4Txt}{HTML}{92400E}
\definecolor{tC5Fill}{HTML}{FEF9C3}  \definecolor{tC5Str}{HTML}{EAB308}  \definecolor{tC5Txt}{HTML}{854D0E}
\definecolor{tC6Fill}{HTML}{FBCFE8}  \definecolor{tC6Str}{HTML}{EC4899}  \definecolor{tC6Txt}{HTML}{831843}
\definecolor{tC7Fill}{HTML}{F5D0FE}  \definecolor{tC7Str}{HTML}{C026D3}  \definecolor{tC7Txt}{HTML}{701A75}
\definecolor{tC8Fill}{HTML}{E9D5FF}  \definecolor{tC8Str}{HTML}{9333EA}  \definecolor{tC8Txt}{HTML}{581C87}
\definecolor{cGray}       {HTML}{94A3B8}
\definecolor{cGrayL}      {HTML}{CBD5E1}
\definecolor{cDark}       {HTML}{334155}
\definecolor{cBlueStr}    {HTML}{3B82F6}
 \begin{figure}[h]

\begin{tikzpicture}[
  x=1cm, y=-1cm,
  tokbase/.style={
    thick, rounded corners=5pt,
    minimum height=0.58cm, minimum width=0.72cm,
    font=\small\bfseries,
    inner xsep=3pt, inner ysep=2pt,
  },
  gentok/.style={
    thick, rounded corners=5pt,
    fill=tC5Fill, draw=tC5Str, text=tC5Txt,
    minimum height=0.52cm, minimum width=0.42cm,
    font=\small\bfseries, inner sep=1pt,
  },
  lmbox/.style={
    draw=cGray, fill=white, thick, rounded corners=4pt,
    minimum width=0.90cm, minimum height=0.78cm, inner sep=2pt,
  },
  kvcell/.style 2 args={
    preaction={fill=#1, fill opacity=0.65},
    pattern=north east lines, pattern color=#2,
    draw=#2, line width=0.5pt,
    minimum width=0.26cm, minimum height=0.26cm, inner sep=0pt,
  },
  kvcellDead/.style 2 args={
    fill=white,
    draw=cGrayL, line width=0.4pt, dashed,
    minimum width=0.26cm, minimum height=0.26cm, inner sep=0pt,
  },
  arr/.style    ={-{Stealth[length=4.5pt,width=4pt]}, thick, color=cDark},
  arrbig/.style ={-{Stealth[length=7pt,width=6pt]}, line width=1.3pt, color=cDark},
  update/.style ={-{Stealth[length=4pt,width=4pt]}, thick, color=cBlueStr, dashed},
  reward/.style ={-{Stealth[length=5pt,width=4.5pt]}, line width=1.1pt, color=cGreen},
]
 
\def\yTok{1.25}
\def\yLa{1.95}   \def\yLb{2.40}   \def\yLc{2.85}
\def\yMid{2.40}
\def\yGen{-0.15}

\def\sxA{1.35}   \def\sxB{2.20}   \def\sxC{3.02}
\def\sxD{3.95}   \def\sxE{4.95}
 
\node[tokbase, fill=tC0Fill, draw=tC0Str, text=tC0Txt]  at (\sxA, \yTok) {Q:};
\node[tokbase, fill=tC1Fill, draw=tC1Str, text=tC1Txt, minimum width=0.78cm] at (\sxB, \yTok) {x+6};
\node[tokbase, fill=tC2Fill, draw=tC2Str, text=tC2Txt, minimum width=0.60cm]  at (\sxC, \yTok) {=7};
\node[tokbase, fill=tC3Fill, draw=tC3Str, text=tC3Txt,
      minimum width=1.00cm, font=\scriptsize\bfseries] at (\sxD, \yTok) {\texttt{<think>}};
\node[tokbase, fill=tC4Fill, draw=tC4Str, text=tC4Txt]  at (\sxE, \yTok) {hm};
 
\node[font=\scriptsize, text=cGray, anchor=east, inner sep=2pt] at (0.95, \yLa) {Layer 1};
\node[font=\scriptsize, text=cGray, anchor=east, inner sep=2pt] at (0.95, \yLb) {Layer 2};
\node[font=\scriptsize, text=cGray, anchor=east, inner sep=2pt] at (0.95, \yLc) {Layer 3};
 
\newcommand{\colcache}[3]{%
  \node[kvcell={#2}{#3}] at (#1, \yLa) {};
  \node[kvcell={#2}{#3}] at (#1, \yLb) {};
  \node[kvcell={#2}{#3}] at (#1, \yLc) {};%
}
\colcache{\sxA}{tC0Fill}{tC0Str}
\colcache{\sxB}{tC1Fill}{tC1Str}
\colcache{\sxC}{tC2Fill}{tC2Str}
\colcache{\sxD}{tC3Fill}{tC3Str}
\colcache{\sxE}{tC4Fill}{tC4Str}

 \node[font=\Large\bfseries, text=cDark, anchor=south]
at (8.5, 0.35) {Neural Garbage Collection};
  
\draw[cDark, very thick, rounded corners=6pt] (1.05, 1.76) rectangle (5.225, 3.08);
\node[font=\small, text=cDark, anchor=north, inner sep=2pt] at (3.225, 3.22) {KV cache};
 
\draw[arr] (5.25, \yMid) -- (5.55, \yMid);

\node[lmbox, font=\normalsize, text=cDark, align=center, inner sep=2pt]
  (lmev) at (6.00, \yMid) {$\pi_\theta$\\[-4pt] \scriptsize LM};
 
\draw[arr] (6.50, \yMid) -- (6.80, \yMid);
 
\def\bxA{7.05}  \def\bxB{7.45}  \def\bxC{7.85}  \def\bxD{8.25}  \def\bxE{8.65}
\def\barW{0.12}
 
\foreach \yy in {\yLa, \yLb, \yLc} {
  \draw[cGray, line width=0.5pt, opacity=0.5] (6.85, {\yy+0.16}) -- (8.85, {\yy+0.16});
}
 
\newcommand{\barz}[6]{%
  \draw[fill=#5, draw=#6, line width=0.4pt, fill opacity=#4, draw opacity=#4]
    ({#1-\barW/2}, {#2+0.16-#3}) rectangle ({#1+\barW/2}, {#2+0.16});%
}
 
\barz{\bxA}{\yLa}{0.36}{1}  {tC0Fill}{tC0Str}
\barz{\bxA}{\yLb}{0.11}{0.4}{tC0Fill}{tC0Str}
\barz{\bxA}{\yLc}{0.36}{1}  {tC0Fill}{tC0Str}
 
\barz{\bxB}{\yLa}{0.11}{0.4}{tC1Fill}{tC1Str}
\barz{\bxB}{\yLb}{0.36}{1}  {tC1Fill}{tC1Str}
\barz{\bxB}{\yLc}{0.11}{0.4}{tC1Fill}{tC1Str}
 
\barz{\bxC}{\yLa}{0.36}{1}  {tC2Fill}{tC2Str}
\barz{\bxC}{\yLb}{0.36}{1}  {tC2Fill}{tC2Str}
\barz{\bxC}{\yLc}{0.11}{0.4}{tC2Fill}{tC2Str}
 
\barz{\bxD}{\yLa}{0.11}{0.4}{tC3Fill}{tC3Str}
\barz{\bxD}{\yLb}{0.36}{1}  {tC3Fill}{tC3Str}
\barz{\bxD}{\yLc}{0.36}{1}  {tC3Fill}{tC3Str}
 
\barz{\bxE}{\yLa}{0.36}{1}  {tC4Fill}{tC4Str}
\barz{\bxE}{\yLb}{0.11}{0.4}{tC4Fill}{tC4Str}
\barz{\bxE}{\yLc}{0.36}{1}  {tC4Fill}{tC4Str}
 
\node[font=\small\bfseries, text=cDark, anchor=center, align=center, inner sep=2pt]
  at (7.85, 1.15) {(1) Score KV cache};
 
\node[font=\small, text=cDark, anchor=north, inner sep=2pt] at (7.85, 3.22) {per-layer softmax scores};
 
\draw[arrbig, decorate, decoration={snake, amplitude=0.6mm, segment length=2.5mm, post length=2mm}]
  (9.00, \yMid) -- (10.35, \yMid);
 
\node[font=\small\bfseries, text=cDark, anchor=center, align=center, inner sep=2pt]
  at (11.4, 1.15) {(2) Sample cache evictions};
 
\def\pxA{10.60}  \def\pxB{11.00}  \def\pxC{11.40}  \def\pxD{11.80}  \def\pxE{12.20}
 
\node[kvcell={tC0Fill}{tC0Str}]  at (\pxA, \yLa) {};
\node[kvcellDead={}{}]           at (\pxB, \yLa) {};
\node[kvcell={tC2Fill}{tC2Str}]  at (\pxC, \yLa) {};
\node[kvcellDead={}{}]           at (\pxD, \yLa) {};
\node[kvcell={tC4Fill}{tC4Str}]  at (\pxE, \yLa) {};
 
\node[kvcellDead={}{}]           at (\pxA, \yLb) {};
\node[kvcell={tC1Fill}{tC1Str}]  at (\pxB, \yLb) {};
\node[kvcell={tC2Fill}{tC2Str}]  at (\pxC, \yLb) {};
\node[kvcell={tC3Fill}{tC3Str}]  at (\pxD, \yLb) {};
\node[kvcellDead={}{}]           at (\pxE, \yLb) {};
 
\node[kvcell={tC0Fill}{tC0Str}]  at (\pxA, \yLc) {};
\node[kvcellDead={}{}]           at (\pxB, \yLc) {};
\node[kvcellDead={}{}]           at (\pxC, \yLc) {};
\node[kvcell={tC3Fill}{tC3Str}]  at (\pxD, \yLc) {};
\node[kvcell={tC4Fill}{tC4Str}]  at (\pxE, \yLc) {};
 
\draw[cDark, very thick, rounded corners=6pt] (10.42, 1.76) rectangle (12.42, 3.08);
\node[font=\small, text=cDark, anchor=north, inner sep=2pt] at (11.42, 3.22) {pruned cache};

\node[kvcellDead={}{}] at (11.00, 3.9) {};
\node[font=\scriptsize, text=cGray, anchor=west, inner sep=2pt] at (11.15, 3.9) {evicted};
 
\draw[arr] (12.55, \yMid) -- (12.95, \yMid);
 
\node[lmbox, font=\normalsize, text=cDark, align=center, inner sep=2pt]
  (lmgen) at (13.45, \yMid) {$\pi_\theta$\\[-4pt] \scriptsize LM};
 
\def\yGenRow{\yMid}
 
\def\gxA{14.65}  
\def\gxB{15.12}  
\def\gxC{15.65}
\node[gentok, fill=tC5Fill, draw=tC5Str, text=tC5Txt] (gx)  at (\gxA, \yGenRow) {x};
\node[gentok, fill=tC6Fill, draw=tC6Str, text=tC6Txt] (geq) at (\gxB, \yGenRow) {=};
\node[gentok, fill=tC7Fill, draw=tC7Str, text=tC7Txt, minimum width=0.52cm] (g5)  at (\gxC, \yGenRow) {1};
 
\draw[arr] (lmgen.east) -- ([xshift=-2pt]gx.west);
 
\node[font=\small\bfseries, text=cDark, anchor=center, align=center, inner sep=2pt]
  at (15.10, 1.15) {(3) Sample next token };
 
\node[
  minimum width=0.72cm, minimum height=0.72cm,
  font=\LARGE\bfseries, text=cGreen, inner sep=0pt,
] (ck) at (16.5, \yGenRow) {$\checkmark$};

\node[font=\small\bfseries, text=cGreenTxt, anchor=north, align=center, inner sep=2pt]
  (rewlbl) at ([xshift=0pt, yshift=-2pt]ck.south) {Task\\[-2pt]reward};
  
\node[font=\large, text=cBlueStr, inner sep=2pt,
      minimum width=0.90cm, minimum height=0.78cm]
  (grad) at (16.5, 4.70) {$\nabla_\theta \mathcal{L}$};
 
 \draw[reward] (rewlbl.south) -- (grad.north);
 
\draw[update] (grad.west) -- (13.45, 4.70) -- (lmgen.south);
\node[font=\scriptsize, text=cBlueStr, anchor=south, align=center, inner sep=1pt]
  at (14.73, 4.70) {policy gradient\\through sampled tokens};
 
\draw[update] (grad.west) -- (6.00, 4.70) -- (lmev.south);
\node[font=\scriptsize, text=cBlueStr, anchor=south, align=center, inner sep=1pt]
  at (8.3, 4.70) {policy gradient through sampled evictions};
\end{tikzpicture}
\vskip 1.5em

\begin{minipage}[t]{0.5\linewidth}
  \centering
\includegraphics[height=4.4cm]{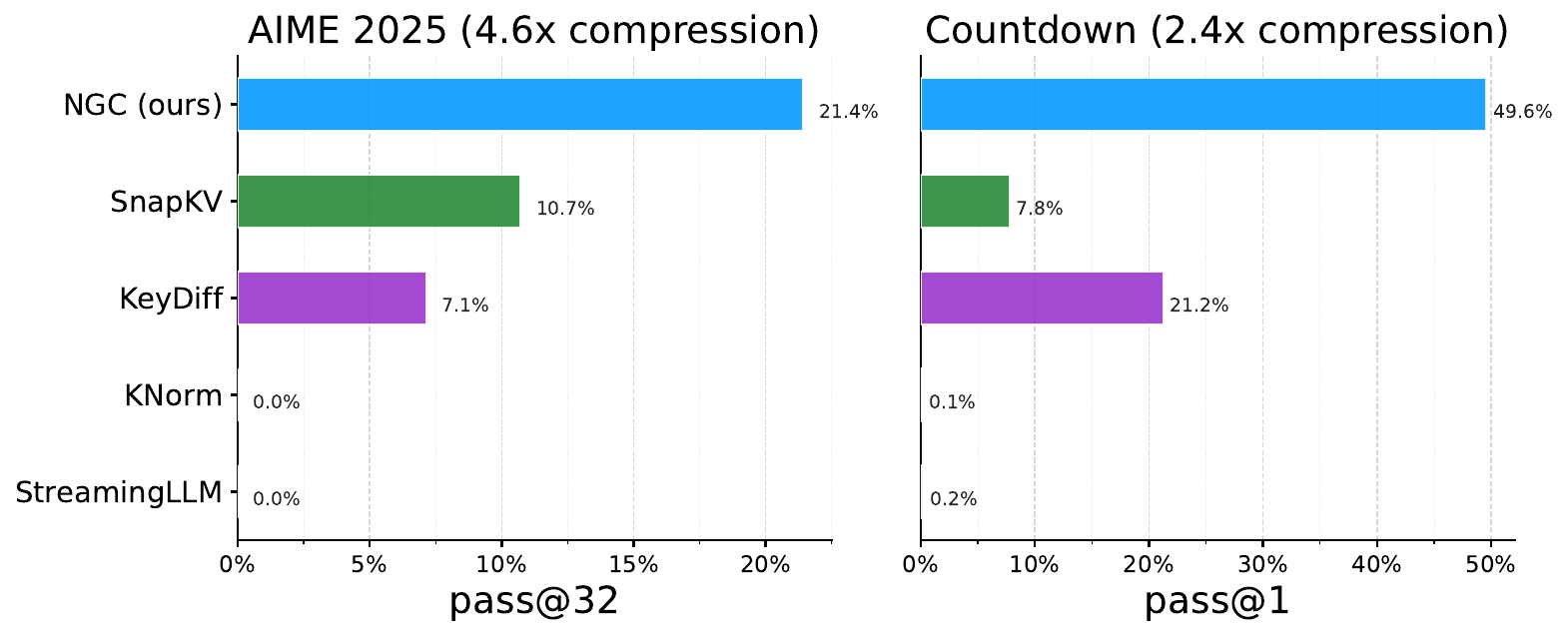}
\end{minipage}%
\begin{minipage}[t]{0.5\linewidth}
  \raggedleft
 \includegraphics[height=4.4cm]{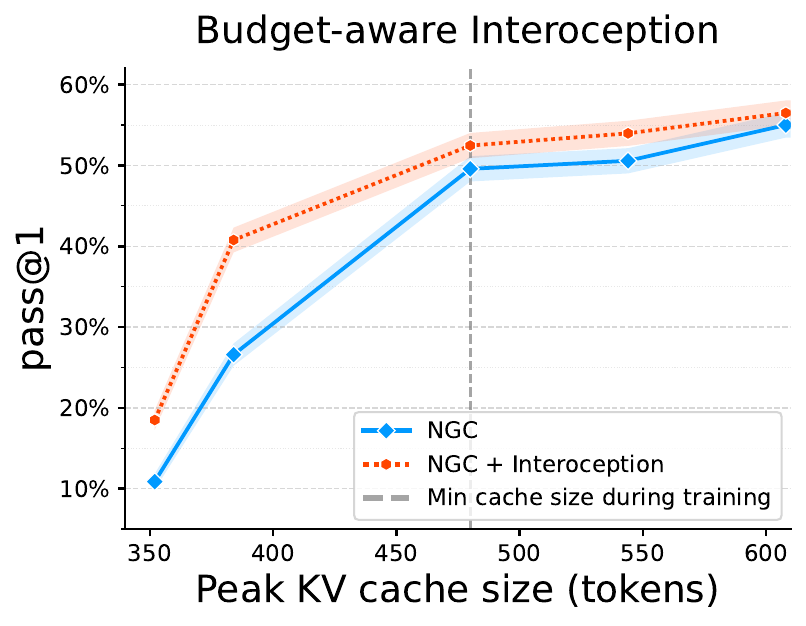}
\end{minipage}

\label{fig:method}
\caption{\textbf{Neural Garbage Collection: Learning to Forget while Learning to Reason.}
\textbf{(top)}
As the language model (LM) generates its chain-of-thought, it periodically enters an {eviction round}: (1) the LM scores KV cache entries via softmax, (2) samples cache evictions using those scores via Gumbel top-$k$, and (3) samples the next token conditioned only on the pruned cache.
Since both tokens and cache evictions are discrete actions sampled from the LM, we can \emph{jointly} train the LM to reason and
manage its memory via \emph{end-to-end reinforcement learning} from
outcome-based task reward alone (e.g., binary correctness) — no proxy
objectives, no supervised finetuning.
\textbf{(bottom left)} NGC substantially outperforms baseline eviction methods at a 50\% eviction rate per round, corresponding to 2.4x reduction in peak KV cache size for Countdown and 4.6x for AIME 2025.
\textbf{(bottom right)} 
To give the LM explicit awareness of its resource constraint before it reasons, we include the eviction rate in the prompt during training — a technique we call \emph{budget-aware interoception}.
This technique improves generalization across cache budgets at test time.
}
\end{figure}

A central lesson of deep learning is that desirable properties emerge from end-to-end optimization pressure rather than being built in manually~\citep{sutton2019bitter}. 
Yet, when it comes to making models more efficient, practitioners rely on deep domain knowledge to design architectures~\citep{dao2024mamba2, ainslie2023gqatraininggeneralizedmultiquery}, systems-level optimizations~\citep{dao2022flashattention, kwon2023vllm}, and inference-time heuristics~\citep{li2024snapkv}. 
We ask whether efficiency itself can be treated as a model capability. 
Under this perspective, the same end-to-end optimization that drives capability can also drive efficiency, enabling an additional dimension of self-improvement~\citep{silver2025experience} in which models become more capable but also more efficient.

A setting where this approach is particularly promising is the chain-of-thought reasoning paradigm~\citep{zelikman2022star} that has driven dramatic advances in model capability~\citep{openai2024learningreason,guo2025deepseekr1}: by thinking longer before answering, models can solve problems that would otherwise be out of reach.
However, these thinking traces rapidly grow the KV cache, creating a bottleneck to further scaling along this dimension.
But paying this cost is not fundamental to reasoning itself — after all, humans reason over long horizons despite severely limited working memory: instead, it is an artifact of how we manage the KV cache. 
Much of the KV cache consists of transient information: intermediate steps that become irrelevant as reasoning progresses or scratch work on sub-problems where only the answer matters downstream.
This suggests that actively managing the KV cache during reasoning need not degrade performance, and may enable models to reason over longer horizons than current resource constraints permit.

Cache management, so far, relies on fixed heuristics or proxy objectives. 
Heuristic eviction policies remove cache entries based on criteria such as attention weights or recency~\citep{xiao2024streamingllm, li2024snapkv, park2025keydiff}.
Compression methods learn compact representations of the model's KV cache using hand-tailored training objectives like reconstruction or distillation losses~\citep{zweiger2026fastkvcompactionattention,monea2025breadcrumbsreasoningmemoryefficientreasoning}, and summarization-based approaches use elaborate prompts that specify what the model should preserve~\citep{vajipey2025simple}.
Relying on fixed rules or proxy objectives for these decisions is anachronistic for reasoning models, whose capabilities emerge from end-to-end optimization against outcome-based task reward.
And since the contents of the KV cache arise from the model's own reasoning process, the model is naturally positioned to manage its own cache. 
If a model can learn to reason, why can't it learn to forget?

To that end, we introduce \textbf{Neural Garbage Collection (NGC)}, in which a language model learns to forget while learning to reason. 
As the model generates a chain-of-thought, it periodically pauses, computes a
softmax over its KV cache entries, samples which to evict, and continues
reasoning conditioned on the pruned cache.
Since this eviction decision is discrete, prior work resorts to differentiable relaxations~\citep{ancucki2025inferencetime} or separate supervised training for those decisions using auxiliary objectives~\citep{liu2025deepseekv3, chen2026msa}. 
Our key insight is that, because we already train reasoning models against
task reward with reinforcement learning — where tokens are discrete actions
sampled from the LM — we can cast cache eviction as \emph{another discrete action}
sampled from the LM and optimize it within the same framework.
A single outcome-based task reward \emph{jointly} trains the model to reason and manage its own memory end-to-end: what the model evicts shapes what it remembers, what it remembers shapes its reasoning, and the correctness of that reasoning trains both the eviction decisions and reasoning tokens that produced it.
Our key methodological contribution is to show that both types of decisions can be optimized from a \emph{single learning signal} — the outcome-based task reward — without auxiliary objectives or separate training stages. 
This follows the \emph{tabula rasa} spirit of AlphaZero~\citep{silver2017masteringchessshogiselfplay}: end-to-end optimization pressure alone guides how the model reasons and what it forgets.

We evaluate NGC on  Countdown~\citep{gandhi2024stream,gandhi2025cognitive}, and on math tasks AMC, and AIME.
On Countdown, NGC more than doubles the accuracy of the next-best baseline (49.6\% vs 21.2\%) at a 2.4x reduction in peak cache size.
We then validate NGC at larger scale by training on \text{DAPO-17k}~\citep{yu2025dapo}.
On mathematical reasoning tasks (AMC, AIME), we show that NGC maintains strong performance at 2-3x reductions in peak KV cache size and outperforms all baseline cache management methods.

\section{Related work}
\paragraph{Language Model Self-Improvement.}
A growing body of work explores language models that self-improve.
STaR \citep{zelikman2022star} demonstrated that LMs can bootstrap reasoning ability by fine-tuning on self-generated chains-of-thought filtered by correctness.
Subsequent work extended this paradigm to more general latent reasoning and iterative self-improvement, including Quiet-STaR \citep{zelikman2024quietstar}, ReST \citep{gulcehre2023rest}, and Self-Rewarding Language Models \citep{yuan2024selfrewarding}.
The reinforcement learning from verifiable rewards (RLVR) paradigm \citep{guo2025deepseekr1, openai2024learningreason} showed that self-improving reasoning scales: models trained on outcome supervision can acquire sophisticated reasoning capabilities.
NGC is inspired by this line of work, but targets a different dimension of self-improvement: rather than improving task performance, the model learns to \emph{improve its own efficiency}. 
This reframes efficiency as a capability: models typically become more capable at the cost of being slower and more expensive to run, whereas a model that improves its efficiency, in principle, becomes \emph{cheaper} as it improves. 
Crucially, many forms of resource management (\eg,eviction, routing) naturally take the form of discrete decisions, making them directly amenable to the same reinforcement learning with verifiable rewards (RLVR) framework used to train modern reasoning models.

\paragraph{Resource Rationality, Bounded Rationality, Metacognition.}
A long tradition in cognitive science and AI studies intelligent behavior under computational constraints, viewing agents as trading off task performance against the cost of computation \citep{15ae9893-24f8-3286-b772-e96b33422e95,horvitz1989reasoning,horvitz1990computation,russell1991do}.
Resource rationality \citep{lieder2020resource} formalizes this perspective by treating cognition as the optimization of task performance subject to resource costs.
This work has primarily used resource rationality as a normative framework to explain behavior, rather than as a training objective to produce it in language models.
A related tradition studies \emph{metacognition}---the ability of an agent to monitor and regulate its own cognitive processes. 
Recent work explores metacognitive prompting and \emph{metacognitive reuse}, where models reflect on prior reasoning traces to extract reusable strategies \citep{didolkar2025metacognitivereuseturningrecurring}. 
NGC teaches a model to regulate its own \emph{current} cognitive state by deciding what information to retain under a memory budget. This makes NGC a form of online, prospective metacognition: the model learns not just to reason, but to manage the memory on which its reasoning depends.
\paragraph{Sparse attention, alternative architectures, and conditional computation.}
Several lines of work improve efficiency by changing how transformers use compute. 
Sparse attention methods reduce the cost of attending over long contexts by selecting only a subset of keys per query~\citep{yuan2025native, liu2025deepseekv3, chen2026msa}. 
These methods use attention scores to determine which entries are \emph{read} at a given step, but they do not reduce the size of the cache itself; memory still grows without bound. 
NGC uses a related scoring mechanism but with fundamentally different semantics: its decisions determine which entries are removed indefinitely, permanently reshaping the context available to future tokens, and providing a practical mechanism for keeping the KV cache constant size.
A related line of work uses \emph{conditional computation} to improve efficiency by routing tokens to only a subset of model parameters, as in mixture-of-experts (MoE) architectures~\citep{shazeer2017outrageously, fedus2022switch, liu2025deepseekv3}. 
These models show that routing decisions can themselves be learned, but they target a different resource and training regime: MoEs sparsify \emph{parameter compute} through architectural routing whereas NGC sparsifies \emph{memory} during inference by deciding what information to retain under an explicit budget.
State-space models~\citep{gu2023mamba, dao2024mamba2} achieve efficiency through different computational primitives, but require pre-training from scratch and cannot fully leverage the thriving hardware and software ecosystem specific to transformers. 
By contrast, NGC requires no modification to the base transformer architecture, repurposing the existing parameters to perform eviction, and can be applied directly during post-training. 

Most crucially, what distinguishes NGC from prior approaches is
\emph{how} its discrete decisions are trained.
In methods such as DeepSeek Sparse Attention~\citep{deepseekai2024deepseekv32},
the indexer, which selects which tokens the model attends to, and the main model are not trained under a unified objective.
Instead, the indexer is first warm-started to imitate dense attention via a
KL-divergence loss against the attention scores, and its input is then explicitly detached from the computational graph throughout sparse
training: the indexer continues to receive signal only from the KL-divergence term, while the
main model is optimized only by the language modeling loss~\citep{deepseekai2024deepseekv32}.
The indexer is a separate module trained by the designer to recognize importance according to a proxy criterion, then installed into the
system while the main model optimizes around it.
NGC instead trains eviction under a unified objective: the same reward signal
that shapes reasoning also shapes how it remembers. 
Moreover, NGC does not require a separate warm-up stage.
Likewise, conditional-computation methods such as mixture-of-experts learn routing decisions, but their routers are typically trained with standard backpropagation-based objectives and load-balancing schemes rather than from downstream task reward \citep{fedus2022switch}.
In contrast, NGC treats discrete choices that modulate memory use as a first-class action and optimizes them with reinforcement learning.

\paragraph{KV cache compression.}
KV cache compression methods reduce memory by evicting or compressing cache entries during generation.
Static methods apply fixed eviction rules: retaining only the most 
recent tokens \citep{xiao2024streamingllm} or tokens with high attention scores 
\citep{li2024snapkv}.
More recent methods use richer proxies for importance such as diversity \citep{park2025keydiff} or optimization objectives to compress the cache~\citep{ eyuboglu2026cartridges, mu2024gist, zweiger2026fastkvcompactionattention}.
Breadcrumbs~\citep{monea2025breadcrumbsreasoningmemoryefficientreasoning} and Memento~\citep{5memento} compress the KV cache during reasoning, but via a fundamentally different mechanism. Breadcrumbs trains a compressed student policy to match an uncompressed teacher via token-level KL divergence, distilling from the teacher's rollouts.
Memento teaches models to produce compressed summaries via SFT on reasoning traces interleaved with summaries generated by a strong teacher prompted with hand-designed rubrics. 
Both methods inject inductive bias about what good compression looks like—Breadcrumbs via the frozen teacher's reasoning style, Memento via the prompt author's rubrics—rather than letting task reward determine what to remember. 
NGC takes a conceptually different approach and jointly learns how to manage memory and reason, end-to-end from outcome-based task reward alone.  
This requires no annotation pipeline, no teacher model, no supervised data, and no hand-designed objectives.

\section{Problem setting: resource use as a learned capability}
\label{sec:problem}

Today, practitioners manage resource constraints through design choices that span 
the entire modeling stack: architecture, systems-level optimizations, and even scaling laws that account for inference time cost~\citep{hoffmann2022training,sardana2025beyond}. 
We ask whether efficiency can instead be \emph{learned}---treating it as a capability 
the model acquires through end-to-end training.

We model resource-constrained generation as a sequential decision process of horizon $T$, 
in which the model makes two kinds of decisions at each step: a task action $a_t$ 
(\eg a generated token) and a resource-management action $u_t$ (\eg, which KV cache 
entries to evict). Let $z_t \in \mathcal{Z}$ denote the model's internal computational 
state at step $t$, and let
\[
    (a_t, u_t) \sim \pi_{\theta,\phi}(\cdot \mid z_t),
\]
where $\theta$ governs task behavior and $\phi$ governs resource allocation; these parameters can be shared or tied but this is a design choice independent of the general formulation.

A trajectory is
\[
    \tau = (z_t, a_t, u_t)_{t=1}^T.
\]
We optimize expected task reward subject to a resource budget $B$:
\begin{equation}
    \max_{\theta,\phi}\; \mathbb{E}_{\tau \sim \pi_{\theta,\phi}}[\mathit{R}(\tau)]
    \quad \text{subject to} \quad
    \mathbb{E}_{\tau \sim \pi_{\theta,\phi}}[C(z_{1:T})] \le B,
    \label{eq:optimization_objective}
\end{equation}
where $R$ is a task reward (\eg, answer correctness) and 
$C : \mathcal{Z}^T \to \mathbb{R}_{\ge 0}$ measures resource consumption 
(\eg, the mean fraction of KV cache entries retained over the trajectory).
The resource-management decision $u_t$ is typically discrete, which makes it naturally 
amenable to the reinforcement-learning paradigm already used to train reasoning models.
We instantiate this framework for KV-cache management in the next section.

\section{Neural Garbage Collection}
\label{sec:method}
In this section, we describe NGC, our method for training a language model to jointly reason and manage its KV cache via end-to-end reinforcement learning using purely an outcome-based task reward. 
The main deviation from standard RLVR settings is that our action space consists of both eviction decisions and tokens. 
Our key methodological contribution is to show that both types of decisions can be optimized from a single learning signal — the outcome-based task reward — without auxiliary objectives or separate training stages.

\subsection{State space and environment dynamics: grow-then-evict cycle}
In this section, we define the \emph{state space} and \emph{environment dynamics} of our reinforcement learning formulation. 
In language modeling, the state corresponds to the model's context window—represented by its KV cache—while the environment dynamics dictate how this context evolves over time. 
In standard RLVR, environment transitions are strictly monotonic: the LM samples the next token conditioned on all previously generated tokens, steadily growing the KV cache. In NGC, the environment dynamics fundamentally differ because the model actively modifies its own state space by managing its own KV cache, in addition to generating tokens.

Specifically, we introduce a ``grow-then-evict'' dynamic. Periodically, at a set of discrete eviction steps, the LM selects cache entries to keep. Every $\delta$ tokens, an \emph{eviction round} fires: the model produces scores for all cache entries and samples a $(1-\varepsilon)$ fraction of them to keep, where $\varepsilon \in (0,1]$ is the \emph{eviction rate}, and permanently evicts the rest; we will refer to $\delta$ as the \emph{eviction cadence}. Both prefill and generation tokens are eligible for eviction and we evict the same fraction (though not the same entries) in each of the $L$ layers of the transformer. The first eviction round fires when the total number of tokens in the cache including prefill tokens reaches $\delta$. Between eviction rounds, the cache grows as usual as the model generates additional tokens conditioned on the surviving entries. Under these dynamics, the maximum cache size before an eviction round converges to the steady state value $L\frac{\delta}{\varepsilon}$; this is independent of the number of reasoning tokens generated. We formalize this in Section~\ref{appendix:proof} of the Appendix.

\subsection{Extending the action space: block-level evictions via existing attention mechanism}
Having formalized the state space, we now describe how we extend the LM's \emph{action space} to include cache evictions. 
At each eviction round, the action is a subset of cache entries to keep. 
Language models do not natively support this type of action since their action space consists only of tokens. 
In the rest of this section, we describe how we use the transformer's existing attention mechanism to parameterize cache evictions, enabling the model to act on its own KV cache — even though it was never trained, through pre-training or supervised fine-tuning (SFT), to do so.

\paragraph{Parameterizing eviction actions}
To perform KV cache evictions, the LM must be able to evaluate the utility of its cache entries—a task it was never explicitly trained to perform. 
Fortunately, the transformer's native attention mechanism is already well-suited for this role. 
First, after standard pre-training, attention weights naturally encode a prior over which past entries are relevant for ongoing computation. 
This provides an effective initialization, which has been shown to be essential for stable reinforcement learning \citep{gandhi2025cognitive}. 
Second, re-purposing this existing machinery introduces no additional architectural overhead. 
We therefore parameterize eviction decisions directly with the model's attention scores, setting $\phi = \theta$ and introducing no new parameters.

Concretely, at each eviction step, we score the remaining (prefix) keys according to how much they are attended to by the $w$ most recent queries.
For layer $\ell$, we use the queries $\mathbf{Q}^{(\ell)}$ associated with the $w$ most recent tokens, compute attention weights from these queries to the prefix keys,
\[
    \mathbf{A}
    = \mathrm{softmax}\!\left(
        \frac{\mathbf{Q}^{(\ell)} {\mathbf{K}^{(\ell)}_{\mathrm{prefix}}}^\top}{\sqrt{d_h}}
      \right),
\]
and average across heads and recent queries to obtain a scalar importance score for each prefix key:
\[
    \psi_t = \frac{1}{H \cdot w} \sum_{h=1}^{H} \sum_{q=1}^{w} A_{q,t}^{(\ell,h)}.
\]

In our experiments, we use $w=5$, so scoring incurs only a small, constant memory overhead.
This construction is inspired by sparse attention methods~\citep{liu2025deepseekv3}, which also use attention scores to decide which keys to read when decoding a token. 
NGC uses the same primitive but with fundamentally different semantics:
rather than deciding which entries are read {at the current step}, these scores decide which entries are \emph{permanently removed} from memory.
Moreover, instead of requiring a separate quantized KV cache, used to decide which actual KV cache entries get used for a decode step, and additional scoring parameters, we re-purpose the attention mechanism into a parameter-free eviction policy.

\paragraph{Coarsening the action space via block-level evictions}
Evicting at the individual key level creates a \emph{credit assignment} problem: the marginal contribution of any single key is difficult to isolate. 
Moreover, the importance of adjacent keys is highly correlated: semantically coherent units such as sub-computations typically occupy contiguous spans of tokens in the chain-of-thought.
Consider a set of surviving keys of length $T$.
We \emph{coarsen} the action space by working at the block level and partitioning the $T$ keys into $N = \lceil T / b \rceil$ blocks of contiguous tokens of size $b$, where the final block may be smaller if $b \nmid T$ ($b$ does not divide $T$).
Let $m_t$ be a mask indicating whether a key $t$ is valid (e.g., non-padding); we omit the layer index for simplicity. 
We aggregate the per-key scores $\boldsymbol{\psi} \in \mathbb{R}^T$ to block-level scores by averaging over valid keys in a block:\[
    s_j = \frac{\sum_{t \in \mathcal{B}_j} m_t \, \psi_t}{\sum_{t \in \mathcal{B}_j} m_t}
\]
yielding $\mathbf{s} \in \mathbb{R}^N$, one score per block.
This dramatically reduces the size of the action space.

While the focus of our work is on the algorithmic properties of NGC rather than hardware optimizations, blockwise selection broadly aligns with hardware principles: GPUs have higher throughput for  contiguous memory accesses, and block-level~\citep{yuan2025native} eviction maps naturally onto paged KV cache systems such as vLLM~\citep{kwon2023vllm}, where the cache is partitioned into fixed-size pages of contiguous keys.

\subsection{Efficiently sampling eviction actions with Gumbel-top-k}
We now define a procedure for tractably \emph{sampling eviction actions} for training. 
KV cache eviction is conventionally performed using deterministic top-k rules~\citep{li2024snapkv,xiao2024streamingllm}.
We instead formulate cache eviction as a \emph{stochastic action} sampled from the LM. 
This formulation enables us to leverage policy gradient methods for unbiased gradient estimates: stochastic sampling with exact log-probabilities is precisely the structure these methods require.

Concretely, we need to be able to sample a size-$K$ subset of kept cache entries in a way that exposes a well-defined log-probability of the selected set.
Initially, this seems challenging since subset selection is combinatorial and the partition function is intractable. 
We formulate this task as sequential sampling without replacement using the \emph{Gumbel-top-$k$ trick} \citep{kool2019stochastic, vieira2014gumbel}. 
Conveniently, sampling can be done in parallel: we sample a size-$K$ subset by adding i.i.d.\ Gumbel noise to the block logits and keep the $K$ largest perturbed values.
The log-probability of the selected block sequence $\boldsymbol{\sigma} = (\sigma_1, \dots, \sigma_K)$ admits a closed form
\begin{equation}
    \log p(\boldsymbol{\sigma} \mid \mathbf{s}) 
    = \sum_{j=1}^{K} \left[s_{\sigma_j} 
    - \log \sum_{t \,\notin\, \{\sigma_1, \dots, \sigma_{j-1}\}} 
    \exp(s_t)\right]
\label{eq:gumbel}
\end{equation}
where indices range over blocks.

Each term naively requires recomputing the partition function over the remaining blocks, \ie, computing a \texttt{logsumexp} $K$ times.
We avoid this via a prefix sum trick that tracks the cumulative fraction of removed probability mass.

\subsection{Policy optimization: token and eviction policy gradients}
\label{sec:per-round-inputs}
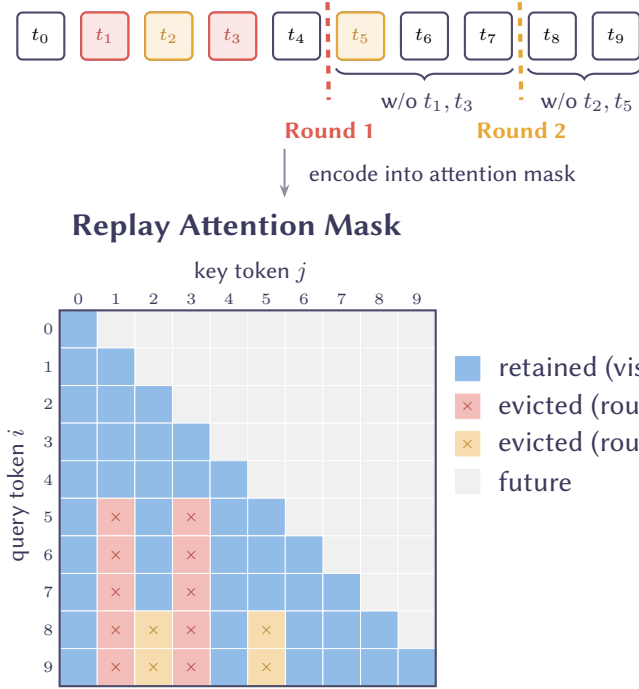
\begin{figure}[ht]
\centering
\begin{tikzpicture}[
    every node/.style={font=\sffamily},
    cell/.style={minimum size=0.52cm, inner sep=0pt, outer sep=0pt, draw=white, line width=0.3pt},
    tok/.style={minimum size=0.625cm, draw=border, thick, rounded corners=2pt,
                inner sep=0.0pt, font=\scriptsize\sffamily\bfseries, fill=white},
    evmark/.style={font=\tiny\sffamily\bfseries, text=white},
]

\def\cs{0.50}  

\node[font=\large\sffamily\bfseries, text=labelc] at (2.6, 8.5)
    {Reasoning with dynamic KV cache management};

\foreach \i in {0,...,9} {
    \node[tok] (t\i) at (\i*0.85, 7.4) {$t_{\i}$};
}

\draw[evictA, ultra thick, dashed] ([xshift=2pt]t4.east) ++(0.015,0.45) -- ++(0.015,-1.3);
\node[font=\small\sffamily\bfseries, text=evictA, anchor=north]
    at ([xshift=2pt, yshift=-1.0cm]t4.east) {Round 1};

\draw[decorate, decoration={brace, amplitude=4pt, mirror}, thick, border]
    ([yshift=-3pt]t5.south west) -- ([yshift=-3pt]t7.south east)
    node[midway, below=5pt, font=\small\sffamily, text=labelc] {w/o $t_1,t_3$};

\draw[evictB, ultra thick, dashed] ([xshift=2pt]t7.east) ++(0.02,0.45) -- ++(0.02,-1.3);
\node[font=\small\sffamily\bfseries, text=evictB, anchor=north]
    at ([xshift=2pt, yshift=-1.0cm]t7.east) {Round 2};

\draw[decorate, decoration={brace, amplitude=4pt, mirror}, thick, border]
    ([yshift=-3pt]t8.south west) -- ([yshift=-3pt]t9.south east)
    node[midway, below=5pt, font=\small\sffamily, text=labelc] {w/o $t_2,t_5$};

\foreach \i in {1,3} {
    \node[evmark, text=evictA] at (t\i) {$\boldsymbol\times$};
    \node[tok, draw=evictA, fill=evictA!15] at (t\i) {};
    \node[font=\scriptsize\sffamily\bfseries, text=evictA!80!black] at (t\i) {$t_{\i}$};
}
\foreach \i in {2,5} {
    \node[tok, draw=evictB, fill=evictB!15] at (t\i) {};
    \node[font=\scriptsize\sffamily\bfseries, text=evictB!80!black] at (t\i) {$t_{\i}$};
}

\draw[-{Stealth[length=5pt]}, thick, border!60]
    (3.24, 5.9) -- (3.24, 5.25)
    node[midway, right=3pt, font=\small\sffamily, text=labelc] {encode into attention mask};

\node[font=\large\sffamily\bfseries, text=labelc] at (2.6, 4.85)
    {Replay Attention Mask};

\def\ox{0.25}
\def\oy{3.75}

\node[font=\small\sffamily, text=labelc, rotate=90] at (\ox - 0.55, \oy - 5*\cs)
    {query token $i$};
\node[font=\small\sffamily, text=labelc] at (\ox + 5*\cs, \oy + 0.5)
    {key token $j$};

\foreach \j in {0,...,9} {
    \node[font=\tiny\sffamily, text=labelc] at (\ox + \j*\cs + 0.5*\cs, \oy + 0.15)
        {$\j$};
}
\foreach \i in {0,...,9} {
    \node[font=\tiny\sffamily, text=labelc] at (\ox - 0.15, \oy - \i*\cs - 0.5*\cs)
        {$\i$};
}

\newcommand{\cellat}[3]{
    \fill[#3] (\ox + #2*\cs, \oy - #1*\cs) rectangle +(\cs, -\cs);
    \draw[white, line width=0.3pt] (\ox + #2*\cs, \oy - #1*\cs) rectangle +(\cs, -\cs);
}

\foreach \i in {0,...,9} {
    \foreach \j in {0,...,9} {
        \pgfmathparse{int(\j > \i)}
        \ifnum\pgfmathresult=1
            \cellat{\i}{\j}{fut}
        \fi
    }
}

\cellat{0}{0}{vis!60}
\foreach \j in {0,1} { \cellat{1}{\j}{vis!60} }
\foreach \j in {0,1,2} { \cellat{2}{\j}{vis!60} }
\foreach \j in {0,1,2,3} { \cellat{3}{\j}{vis!60} }
\foreach \j in {0,1,2,3,4} { \cellat{4}{\j}{vis!60} }
\foreach \j in {0,2,4,5} { \cellat{5}{\j}{vis!60} }
\foreach \j in {0,2,4,5,6} { \cellat{6}{\j}{vis!60} }
\foreach \j in {0,2,4,5,6,7} { \cellat{7}{\j}{vis!60} }
\foreach \j in {0,4,6,7,8} { \cellat{8}{\j}{vis!60} }
\foreach \j in {0,4,6,7,8,9} { \cellat{9}{\j}{vis!60} }

\foreach \i in {5,6,7} {
    \cellat{\i}{1}{evictA!40}
    \cellat{\i}{3}{evictA!40}
}
\foreach \i in {8,9} {
    \cellat{\i}{1}{evictA!40}
    \cellat{\i}{3}{evictA!40}
}
\foreach \i in {8,9} {
    \cellat{\i}{2}{evictB!40}
    \cellat{\i}{5}{evictB!40}
}

\foreach \i in {5,6,7,8,9} {
    \node[font=\tiny, text=evictA!80!black] at (\ox + 1*\cs + 0.5*\cs, \oy - \i*\cs - 0.5*\cs) {$\times$};
    \node[font=\tiny, text=evictA!80!black] at (\ox + 3*\cs + 0.5*\cs, \oy - \i*\cs - 0.5*\cs) {$\times$};
}
\foreach \i in {8,9} {
    \node[font=\tiny, text=evictB!80!black] at (\ox + 2*\cs + 0.5*\cs, \oy - \i*\cs - 0.5*\cs) {$\times$};
    \node[font=\tiny, text=evictB!80!black] at (\ox + 5*\cs + 0.5*\cs, \oy - \i*\cs - 0.5*\cs) {$\times$};
}

\draw[border, thick] (\ox, \oy) rectangle (\ox + 10*\cs, \oy - 10*\cs);

\def\lx{5.5}
\def\ly{2.8}

\fill[vis!60] (\lx, \ly) rectangle ++(0.35, 0.35);
\draw[white, line width=0.3pt] (\lx, \ly) rectangle ++(0.35, 0.35);
\node[font=\sffamily, text=labelc, anchor=west] at (\lx + 0.45, \ly + 0.175) {retained (visible)};

\fill[evictA!40] (\lx, \ly - 0.5) rectangle ++(0.35, 0.35);
\draw[white, line width=0.3pt] (\lx, \ly - 0.5) rectangle ++(0.35, 0.35);
\node[font=\tiny, text=evictA!80!black] at (\lx + 0.175, \ly - 0.325) {$\times$};
\node[font=\sffamily, text=labelc, anchor=west] at (\lx + 0.45, \ly - 0.325) {evicted (round 1)};

\fill[evictB!40] (\lx, \ly - 1.0) rectangle ++(0.35, 0.35);
\draw[white, line width=0.3pt] (\lx, \ly - 1.0) rectangle ++(0.35, 0.35);
\node[font=\tiny, text=evictB!80!black] at (\lx + 0.175, \ly - 0.825) {$\times$};
\node[font=\sffamily, text=labelc, anchor=west] at (\lx + 0.45, \ly - 0.825) {evicted (round 2)};

\fill[fut] (\lx, \ly - 1.5) rectangle ++(0.35, 0.35);
\draw[white, line width=0.3pt] (\lx, \ly - 1.5) rectangle ++(0.35, 0.35);
\node[font=\sffamily, text=labelc, anchor=west] at (\lx + 0.45, \ly - 1.325) {future};

\end{tikzpicture}
\vspace{1em}
\caption{
\textbf{Replay masks enable efficient end-to-end RL training.}
During reasoning, the language model 
samples eviction decisions that dynamically change its KV cache.
We can perform the policy gradient update efficiently and correctly using \emph{replay attention masks.}
The eviction decisions can be captured by attention masks that reproduce the visibility patterns over previous tokens induced by the eviction decisions.
We then use these masks to compute log-probabilities in parallel using a single forward pass over all the tokens.
\textbf{(top)}~The model evicts cache entries at fixed intervals.
({\color{evictA}\textbf{round\,1}}: evicts $t_1, t_3$;\;
 {\color{evictB}\textbf{round\,2}}: evicts $t_2, t_5$).
For illustrative purposes, we show the schematic for a single layer, but in practice we evict separately across layers.
\textbf{(bottom)}~The resulting replay mask: 
row $t$ marks which keys were visible when token $t$ was generated.  A single forward pass with this mask reproduces the 
next-token distributions from generation in parallel.
Without the replay mask, token log-probabilities are computed under a richer context than the model actually saw during generation, introducing a systematic off-policyness that causes training collapse.
}
\label{replay-mask}
\end{figure}

With the states, actions, and sampling procedure established, we now formalize the \emph{policy optimization objective}. In standard RLVR, the policy gradient applies solely to the token space. 
Here, we define a unified objective over a \emph{joint} action space consisting of both tokens and discrete cache evictions. 
Concretely, we detail how the single, outcome-based task reward in RLVR can be used to update the model to jointly optimize both its chain-of-thought reasoning and how it manages it own KV cache.

Given a prompt $x$, we sample $G$ trajectories $\{\tau_i\}_{i=1}^G$ from $\pi_\theta$, where each trajectory $\tau_i = \{(o_{i,t}, \boldsymbol{\sigma}_{i,t})\}_{t=1}^{|\tau_i|}$ interleaves token generation $o_{i,t}$ with eviction decisions $\boldsymbol{\sigma}_{i,t}$, where $\boldsymbol{\sigma}_{i,t} = \emptyset$ at timesteps $t$ that do not correspond to an eviction round. 
Each trajectory is scored with a binary task reward $r_i \in \{0,1\}$ computed from the generated tokens $\{o_{i,t}\}$.

Following GRPO~\citep{guo2025deepseekr1}, we compute group-normalized advantages:
\[
\hat{A_i} = r_i - \frac{1}{G}\sum_{j=1}^G r_j.
\]
The only training signal is task accuracy.

We optimize $\pi_{\theta}$ using Dr. GRPO~\citep{liu2025understandingr1zeroliketraining}. 
For the token-level actions, we have a token-level policy gradient:
\begin{equation}
\mathcal{L}_{\text{token}} = -\mathbb{E}\left[\frac{1}{G} \sum_{g=1}^G \sum_{t=1}^{|o_i|} \log \pi_{\theta}(o_{i,t} \mid \mathcal{C}_{i, t})\, \hat{A_i}\right].
\label{eq:token}    
\end{equation}
Here, $\mathcal{C}_{i,t}$ denotes the context available at step $t$ of trajectory $i$ (\ie, the surviving KV cache entries).

For the eviction decisions, we have an eviction decision level policy gradient
\begin{equation}
\mathcal{L}_{\text{mem}} = \sum_{\ell=1}^{L} \mathcal{L}^{\ell}_{\text{mem}}, \qquad \mathcal{L}^{\ell}_{\text{mem}} = -\mathbb{E}\Bigg[\frac{1}{G}\sum_{i=1}^{G} \frac{1}{|\mathcal{T}^\ell_{i}|} \sum_{t \in \mathcal{T}^\ell_{i}} \log \pi_{\theta}(\boldsymbol{\sigma}^\ell_{i,t} \mid \mathcal{H}^\ell_{i,t})\, \hat{A_i} \Bigg]
\label{eq:mem_eq}
\end{equation}
where $L$ is the number of transformer layers, $\mathcal{T}^\ell_i$ is the set of eviction rounds at layer $\ell$ for rollout $i$ that fire before sequence termination, $\boldsymbol{\sigma}^\ell_{i,t}$ is the subset of blocks retained at round $t$, and $\mathcal{H}^\ell_{i,t}$ is the  context used to make eviction decisions at the $\ell$-th layer (layer-$\ell$ queries and alive keys at round $t$); with slight abuse of notation, we write $\pi_\theta(\boldsymbol{\sigma}^\ell_{i,t} \mid \mathcal{H}^\ell_{i,t})$ to denote the probability of the retained subset given by Gumbel top-k in Equation~\ref{eq:gumbel}. 
The inner mean averages over eviction rounds within an individual rollout.
We use the layer index to emphasize that we make separate eviction decisions for each KV entry in each layer rather than making the same decision for all layers for a given token.

Crucially $\hat{A_i}$ is the \emph{same} for both $\mathcal{L}_{\text{token}}$ and $\mathcal{L}_{\text{mem}}$: both learning to generate reasoning tokens and learning to evict come from the same learning signal. 
The total objective is:
\begin{equation}
\mathcal{L} = \mathcal{L}_{\text{token}} +  \mathcal{L}_{\text{mem}}.
\label{eq:loss}    
\end{equation}
We again emphasize that although we call $\mathcal{L}$ a loss, it  is not an auxiliary loss (e.g., an $\ell_1$ sparsity penalty or a KL divergence between teacher and student attention weights~\citep{deepseekai2024deepseekv32}): it introduces no independent training signal beyond the task reward already present in $\hat{A_i}$. It is a scalar whose gradient is the correct policy gradient estimate for both token and eviction decisions. 
The resource constraint in our setting enters by construction (evicting a fixed amount of cache each time) rather than through auxiliary losses.

Because eviction decisions are scored by the model's own queries and keys ($\phi = \theta$), the policy gradient from $\mathcal{L}_{\text{mem}}$ flows into every parameter that shapes $\mathbf{Q}$ and $\mathbf{K}$---which is every parameter the LM uses to reason. Token generation and cache management are therefore treated as a single policy optimized under a single outcome reward.

\subsection{Computing rollout log-probabilities efficiently with replay masks}
\label{sec:replay}
To compute the policy gradients defined above, we need to evaluate the log-probabilities of the exact trajectories sampled during rollouts. 
This section introduces \emph{replay attention masks}, a mechanism to efficiently compute these probabilities.

The key challenge is that because NGC dynamically modifies its own context through cache eviction, it breaks an assumption in standard RLVR: that the model attends to all previous tokens when generating a new token.
Concretely, when the model generates a new token, it does not attend to cache entries it has already evicted.
Therefore, naively computing log-probabilities under a standard causal attention mask would be incorrect and create a form of off-policyness~\citep{zheng2025gspo, liu-li-2025-rl-collapse}.
We use {replay attention masks} to correctly and efficiently compute Equation~\ref{eq:token} (\ie, the log-probability of the generated tokens) during training.
Specifically, we ``replay'' the
rollout log-probabilities under a per-layer attention mask that replicates the exact visibility pattern over past tokens that was induced by the eviction decisions (Figure~\ref{replay-mask}). 
A single forward pass over those tokens with these replay masks gives us the correct log-probabilities.

Another subtle technical challenge is ensuring that gradients flow from $\mathcal{L}_{\text{mem}}$; see Figure~\ref{fig:autograd} for a schematic of the desired autograd graph.  
Eviction decisions at a given round are scored using the queries and keys at that eviction round.  
If those tensors were
simply cached during rollouts and passed directly into
Equation~\ref{eq:mem_eq}, they would be numerically correct but
\emph{detached from the computation graph}, and no gradient
from $\mathcal{L}_\text{mem}$ would reach $\theta$.  
We recompute these inputs in the replay forward pass, so that the values are on the autograd graph.
Using these inputs, we then re-compute the eviction log-probabilities via Equation~\ref{eq:mem_eq},
analogously to how we recompute per-token log-probabilities for
$\mathcal{L}_\text{token}$.
The resulting eviction log-probabilities are therefore live on the
autograd graph, and gradients flow from $\mathcal{L}_\text{mem}$
through the eviction decisions into $\theta$.

Crucially, this replay step introduces no additional overhead beyond standard RL training,
which already requires a forward pass over sampled trajectories.
We only have to store a negligible amount of additional state (a binary mask per layer and per eviction round) to construct the replay masks.

\subsection{Eviction Rate Curriculum}
Finally, we introduce a curriculum to ensure that our model can be trained stably to both optimize its own reasoning and cache management. 
Intuitively, evicting large amounts of the KV cache from the start of training can be destabilizing to the model. 
During training, we smooth the transition to more aggressive eviction rates by slowly increasing the eviction rate using a staircase curriculum over the retention rate $p_0$ (\ie, percentage of kept entries).
More formally, let $\{\rho_0, \rho_1, \ldots, \rho_K\}$ be a sequence of retention rates with $\rho_0 \geq \rho_1 \geq \cdots \geq \rho_K$, and let $\Delta$ denote the number of training steps per stage.
At training step $t$, the current stage index is $\ell = \min(\lfloor t / \Delta \rfloor,\, K)$.
Within stage $\ell < K$, we apply a linear blend toward the next level during the final fraction $\alpha \in (0,1)$ of the stage.
We set $\alpha=0.6$ in our experiments.
Let $s = (t \bmod \Delta) / \Delta$ denote the fractional progress within the stage. Then:
\begin{equation}
    p_0(t) = \begin{cases}
        \rho_\ell & \text{if } s < 1 - \alpha, \\
        \rho_\ell + \dfrac{s - (1-\alpha)}{\alpha}\,(\rho_{\ell+1} - \rho_\ell) & \text{if } s \geq 1 - \alpha.
    \end{cases}
\end{equation}
At the final stage $\ell = K$, we set $p_0(t) = \rho_K$.
See the rightmost panel of Figure~\ref{fig:ngc-training-dynamics} for an illustration of this schedule.

\section{Experiments}
We begin with a controlled study on Countdown, validating and ablating core design choices, and then evaluate on standard math competition benchmarks.

\subsection{Shared experimental setup}
\paragraph{Model and training recipe.}
We train DeepSeek-R1-Distill-Qwen-1.5B~\citep{guo2025deepseekr1} using on-policy Dr. GRPO with a staircase eviction curriculum. We use block size $b=32$ and window size $w=5$.
Training hyperparameters are standard; see Section~\ref{sec:hypers} in the Appendix for details.

\paragraph{Baselines.}
We compare NGC against four inference-time eviction heuristics applied to a model trained with the full KV cache: \textbf{SnapKV}~\citep{li2024snapkv} (attention weights), \textbf{KeyDiff}~\citep{park2025keydiff} (key diversity), \textbf{KNorm}~\citep{devoto2024knorm} (key norm statistics), and \textbf{StreamingLLM}~\citep{xiao2024streamingllm} (sliding window with attention sinks). 

\paragraph{Metrics.}
We report accuracy on held-out test problems. 
At evaluation time, we replace Gumbel-top-$k$ sampling with greedy top-$k$ selection, retaining the $K$ highest-scored blocks deterministically.

To measure memory savings, we report the \textbf{average peak KV cache reduction}: the factor by which a method reduces the \textbf{peak KV cache size} relative to no eviction. Concretely, let $p_i$ denote the prompt length, $c_i^{\text{base}}$ the completion length under the no-eviction baseline, and $c_i^{\text{method}}$ the completion length under the method being evaluated. Let $\text{peak}(p, c, \epsilon, \delta)$ denote the maximum number of KV entries held in memory at any point during generation for a sequence with prompt length $p$, completion length $c$, eviction rate $\epsilon$, and cadence $\delta$. The average peak KV cache reduction is then
\begin{equation*}
    \mathbb{E}\!\left[\frac{\text{peak}(p_i,\, c_i^{\text{base}},\, 0,\, \delta)}{\text{peak}(p_i,\, c_i^{\text{method}},\, \epsilon,\, \delta)}\right],
\end{equation*}
where the expectation is over prompts. A value of 2x means the method uses half the peak memory of the no-eviction baseline.

\subsection{Controlled analysis on Countdown}
\paragraph{Setup.}
We train {DeepSeek-R1-Distill-Qwen-1.5B}~\citep{guo2025deepseekr1} for 250
steps on Countdown~\citep{gandhi2024stream}, a combinatorial arithmetic task requiring
a model to combine randomly drawn numbers via basic operations to reach a target number.
We use a maximum completion length of 1024 tokens, eviction cadence 256, 32 unique prompts
per update step, and 16 rollouts per prompt. Our training set contains 327,680
problems (3 and 4 input numbers); our test set contains 1024 held-out problems.
Countdown is a demanding testbed since {DeepSeek-R1-Distill-Qwen-1.5B}
achieves near-zero accuracy before training; NGC must jointly learn to reason about this task
and manage memory from scratch.
It has also become a standard setting for evaluating LM reasoning at scales accessible to academic labs~\citep{pan2025learning,gandhi2025cognitive, gandhi2024stream}.

\subsubsection{Main results}
\paragraph{NGC outperforms all baselines on test-set.}
\begin{figure}[!ht]
\centering
   \includegraphics[width=1.0\linewidth]{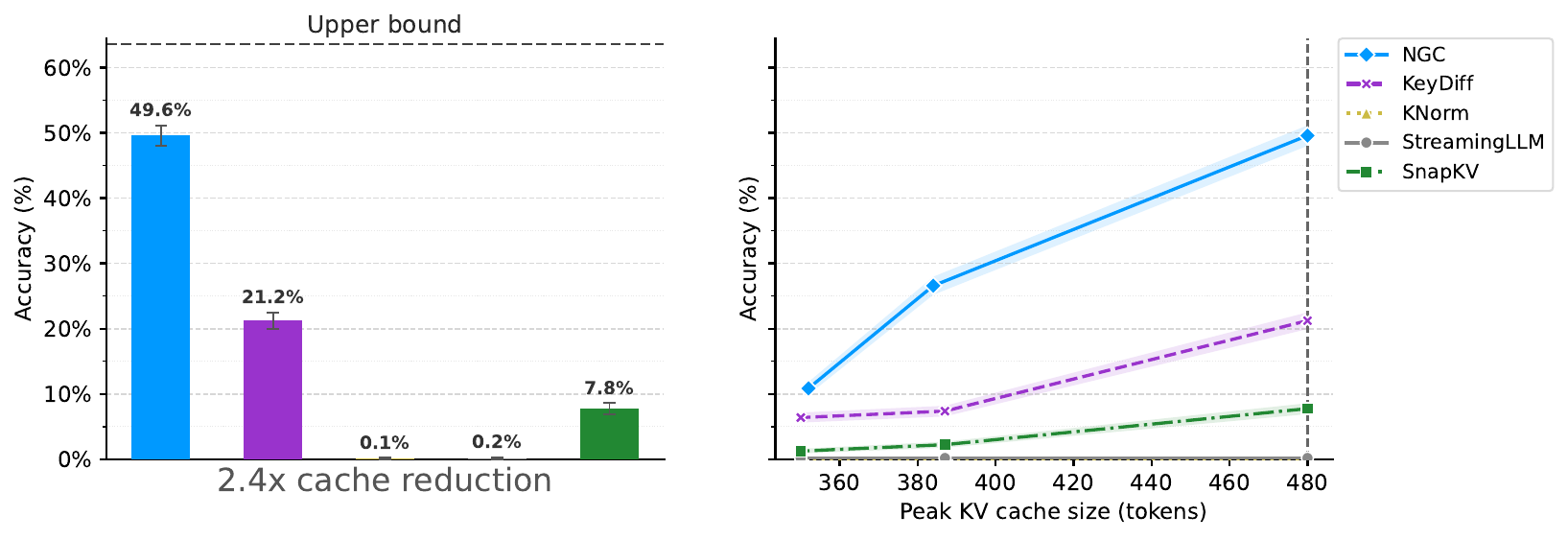}
\caption{\textbf{Comparing NGC against KV cache eviction baselines.} 
\textbf{(left)} NGC significantly outperforms baseline cache eviction methods. Error bars correspond to 1 SE and we report pass@1 accuracies. All methods evict 50\% of entries at each eviction round, corresponding to~2.4x reduction in peak KV cache size.
The dashed horizontal line corresponds to a model trained and evaluated with full cache.
\textbf{(right)}
We also test NGC's generalization to other eviction rates (smaller cache sizes) that it  was not explicitly trained for; the vertical dashed line indicates the minimum cache size during training. NGC pareto-dominates all other methods even at higher compression rates.
}
\label{fig:accuracy_bar}
\end{figure}
We report accuracy on 1024 held-out test problems.
In Figure~\ref{fig:accuracy_bar}, we compare NGC against several baseline methods at a 
50\% eviction rate per eviction round (corresponding to 2.4x reduction in peak cache size). NGC achieves 49.6\% accuracy, substantially outperforming all 
heuristic methods. 
This shows that NGC learns to reason and manage memory jointly, a more challenging task than learning to reason alone.

\paragraph{NGC generalizes to unseen eviction rates.}
In Figure~\ref{fig:accuracy_bar} (right), we evaluate NGC across a range of eviction 
rates more aggressive than those seen during training. 
NGC pareto-dominates all baselines, and its performance degrades gracefully beyond the training distribution until a very low cache size. 

\subsubsection{Ablations}
To understand the benefits of end-to-end training, we consider two ablations.
\paragraph{End-to-end training is essential.}
\begin{figure}[t]
    \centering
    \includegraphics[width=0.4\linewidth]{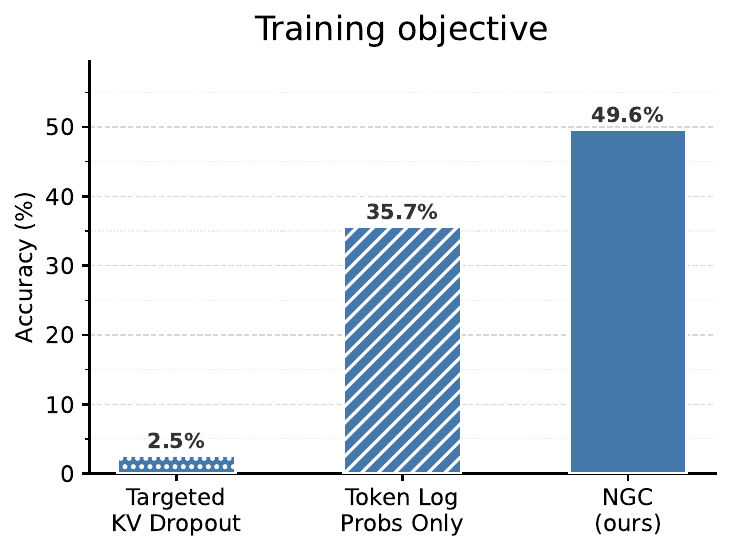}
\caption{\textbf{Ablating NGC design choices.}
We use two ablations to characterize the properties of end-to-end training.
{Targeted KV dropout} evicts cache entries during RL but ignores the off-policyness introduced by evictions.
{Token log-probs only} corrects for
off-policyness via replay masks but drops the eviction decision policy gradient term. 
Both significantly under-perform NGC.
}
\label{fig:ngc-ablations}
\end{figure}

Our first ablation evaluates whether exposing the model to cache eviction during training enables the model to handle evictions at test time.
Concretely, \textbf{Targeted KV dropout} evicts cache entries during training using NGC's scoring mechanism and applies the same curriculum, but optimizes only Equation~\ref{eq:token} with a standard causal mask---ignoring the off-policyness that eviction introduces; see Section~\ref{sec:replay}.
\textbf{Token log-probs only} corrects for this off-policyness using replay masks but drops the eviction policy gradient term $\mathcal{L}_{\text{mem}}$. 
Both significantly under-perform NGC (2.5\% and 35.7\% 
vs.\ 49.6\%), confirming that correct off-policy handling and end-to-end training of both reasoning and eviction are essential for performance.

\paragraph{Training dynamics.}
\begin{figure}[t]
    \centering
    \includegraphics[width=\linewidth]{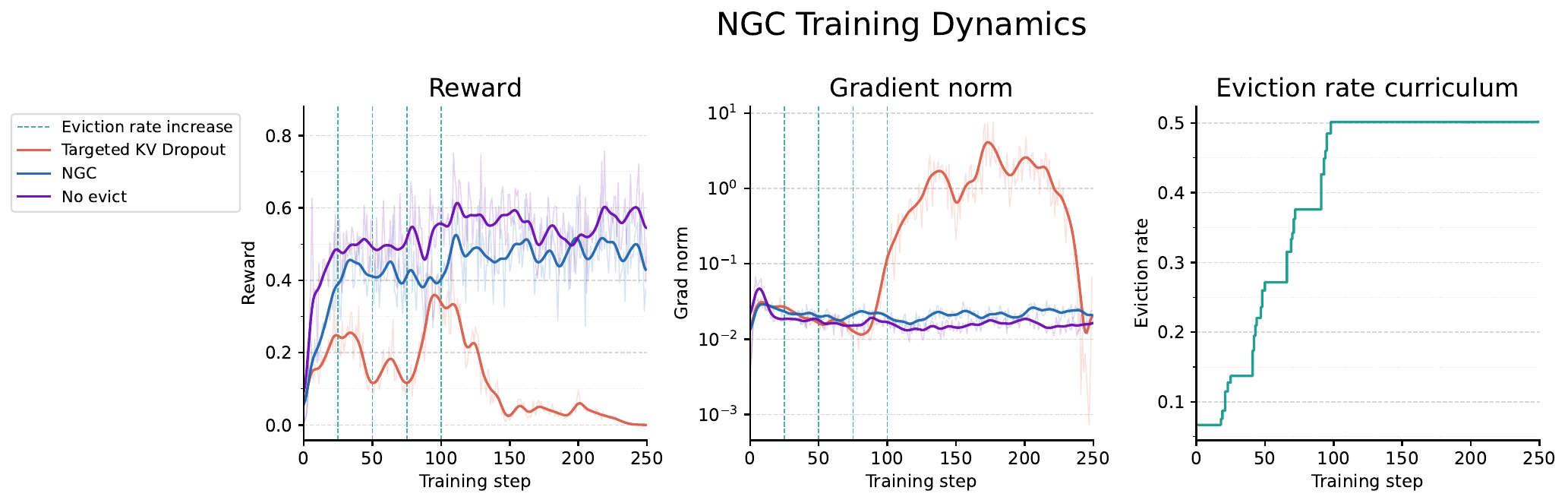}
    \caption{
        \textbf{NGC training dynamics on Countdown.}
        (\textbf{left}) Reward over training: NGC improves steadily but more slowly than training without any eviction (no evict) but eventually reaches a similar level of accuracy. 
        Targeted KV Dropout increases initially at a low eviction rate but collapses after around step~100 due to off-policyness.
        (\textbf{center}) Targeted Dropout exhibits exploding gradient norms coinciding with its reward collapse, whereas NGC and the no-eviction baseline remain stable throughout training.
        (\textbf{right}) The staircase eviction rate curriculum, shared across both NGC and targeted KV dropout, which gradually increases the percentage of KV cache evicted to $50\%$ over~$\sim\!100$ steps.
        Vertical dashed teal lines in the left and center panels mark the steps where the eviction rate begins to increase.
    }
    \label{fig:ngc-training-dynamics}
\end{figure}
Figure~\ref{fig:ngc-training-dynamics} shows reward and gradient norm over training. NGC 
improves steadily but more slowly than training without eviction (\text{no evict}). 
Targeted KV dropout initially improves at low eviction rates but collapses during training around step 100, coinciding with a spike in gradient norm. 
This illustrates the training 
instability created by off-policyness. 

\subsection{Budget-aware Interoception}
\label{sec:interoception}
If resource use is part of the task, then intuitively the model could benefit from ``sensing'' its budget constraints before reasoning. 
We ask whether explicitly conditioning the model on the eviction rate in the prompt improves its accuracy and generalization to various compression rates.
We call this \emph{budget-aware interoception}: like organisms sensing internal state (such as hunger) to maintain homeostasis, the model perceives its memory budget before it reasons.
\paragraph{Setup.}
To implement this, we simply modify the prompt during training and testing: we sample a single eviction rate
$\rho$ for a group of $G$ rollouts, as before, and simply append the same structured tag to every prompt in that group:
 \begin{center}
\ttfamily\small
\textless eviction\_rate\textgreater$\rho$\%\textless/eviction\_rate\textgreater
\end{center}
This approach takes inspiration from the ``difficulty-conditioned'' conjecturer from~\citet{poesia2024learning}. 
At inference time, we vary the eviction rate and compare against a standard NGC model trained under identical conditions but without the
eviction-rate tag in the prompt.

\paragraph{Results.}
\begin{figure}[!ht]
\centering
 \includegraphics[width=0.45\linewidth]{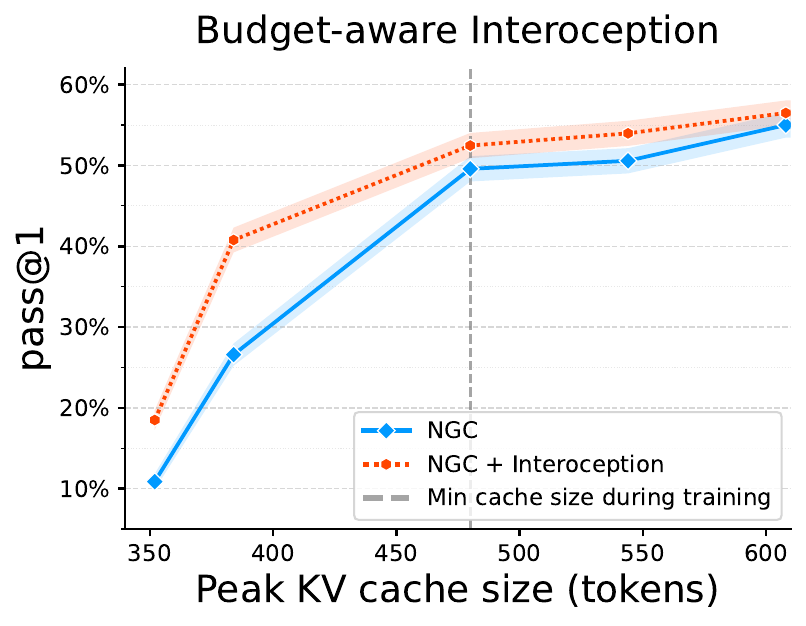}
\caption{\textbf{Budget-aware interoception.}
By training and evaluating the model with the eviction rate $\rho$ in the prompt{\ttfamily\small \textless eviction\_rate\textgreater$\rho$\%\textless/eviction\_rate\textgreater}, we see softened performance degradation as peak KV cache size decreases and generalization to stricter budgets. At aggressive budgets, this gives an 8-13\% boost in performance.
The dashed vertical line represents the minimum cache size used during training. 
}
\label{fig:interoception}
\end{figure}
Figure~\ref{fig:interoception} reports accuracy as a function of the peak cache size.
The interoceptive model matches or exceeds standard NGC across the
full training distribution, and the gap widens at more aggressive eviction
rates which require generalization beyond the training distribution.
The result also illustrates a broader principle: conditioning on resource constraints in the prompt is a simple way to enable a model to have a meta-awareness of its own computational constraints.

\subsection{Mathematical reasoning experiments}
\begin{figure}[t]
    \centering
    \includegraphics[width=\linewidth]{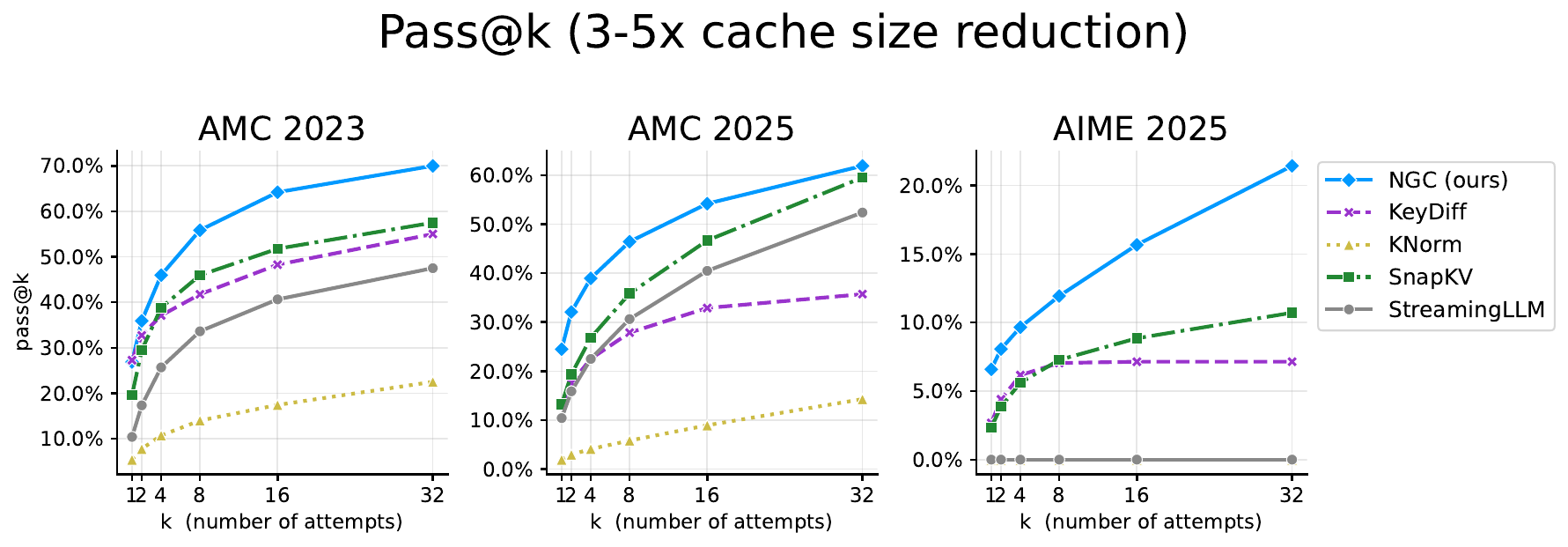}
\caption{\textbf{Pass@$k$ across math reasoning benchmarks.}
NGC consistently outperforms all inference-time eviction baselines across AMC 2023, AMC 2025, and AIME 2025; all methods evict 50\% at each eviction round, corresponding to a 3-5x peak cache size reduction.
}
\label{fig:passatk-all}
\end{figure}
\begin{figure}[t]
    \centering
    \includegraphics[width=\linewidth]{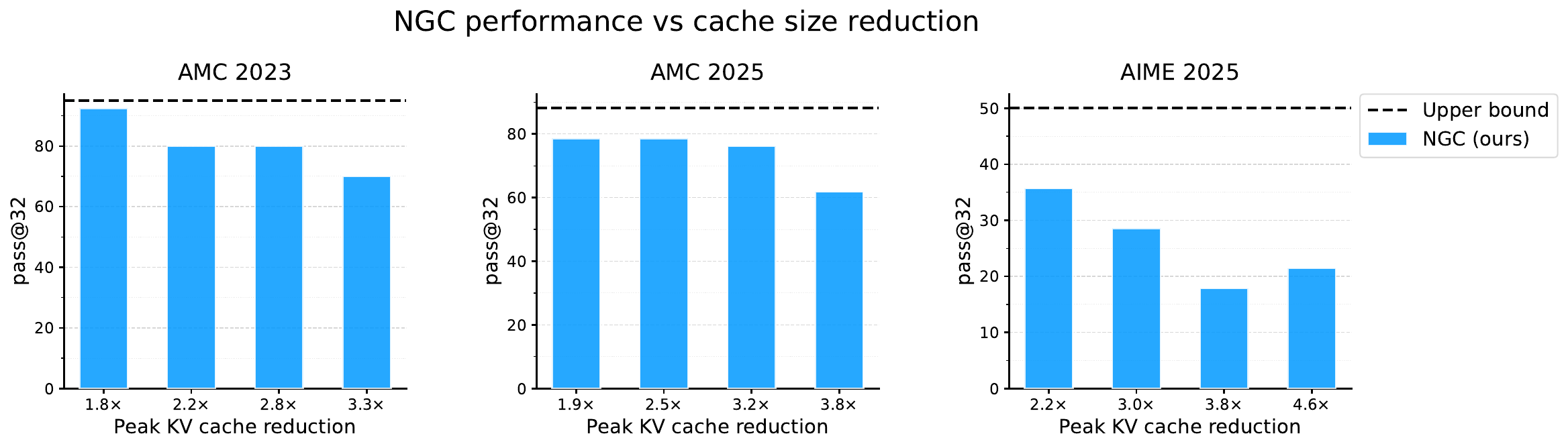}
\caption{\textbf{NGC pass@$32$ across varying cache size reductions.}
For moderate reductions in max cache size 2-3x, NGC can still maintain relatively strong accuracy relative to an upper bound that is trained with no eviction and uses no eviction during inference.
}
\label{fig:compression_ratio}
\end{figure}
To test whether the same approach generalizes to broader reasoning tasks, we now evaluate NGC on standard math competition benchmarks.

\paragraph{Setup.}
We train DeepSeek-R1-Distill-Qwen-1.5B on DAPO-17k~\citep{yu2025dapo}, a large-scale mathematical reasoning dataset spanning a broad range of problem types and difficulties. We evaluate on AMC 2023, AMC 2025, and AIME 2025 to assess whether NGC generalizes beyond its training distribution. 
We use the same baseline eviction approaches as before on top of a model trained without eviction.
We train with an eviction cadence of 350 tokens and a maximum response length of 1050, using 256 prompts per update step and 8 rollouts per prompt for a total of 469 steps.
At test time, we use $\text{top}\_\text{p}=0.95$ and temperature $0.6$ and a max completion length of 3850 tokens.

\paragraph{Results.} 
Figure~\ref{fig:passatk-all} reports pass@k~\citep{chen2021evaluatinglargelanguagemodelstrained} at a 3-5x cache reduction size (corresponding to 50\% eviction rate at every eviction round) across AMC 2023, AMC 2025, and AIME 2025. 
NGC consistently outperforms baselines across all three benchmarks; heuristic baselines can degrade to near-zero accuracy while NGC maintains a reasonable performance.
Moreover, heuristic baselines are inconsistent across datasets. 
The relative ranking of KeyDiff and SnapKV shifts across benchmarks: SnapKV fails 
catastrophically on Countdown but performs reasonably well on math reasoning tasks. 
This inconsistency illustrates why a fixed proxy for importance is fragile — it is task-dependent in ways that are difficult for a practitioner to anticipate.
NGC's consistently strong performance across tasks shows that end-to-end training allows the model to adapt to each task and discover what information to keep in a task-dependent manner.
In Figure~\ref{fig:compression_ratio}, we compare the pass@32 performance across different cache reduction sizes. 
We find that at more moderate (2-3x) cache reduction sizes, NGC can come close to matching the upper bound (no-eviction ceiling).
\section{Conclusion}
We introduced Neural Garbage Collection, a framework in which a language model jointly learns to reason and manage its own KV cache by optimizing outcome-based task reward alone. 
Experiments on Countdown and DAPO-17k show that NGC outperforms heuristic eviction baselines. 
NGC is a first step towards a broader vision: as models become more capable, they can also learn to use their own resources more efficiently. 
Under this view, efficiency is a behavior the model can acquire through the same end-to-end optimization pressure that drives capability. 
This could unlock the possibility of inverting the usual tradeoff between capability and cost.

\section{Acknowledgements}
We especially thank Kanishk Gandhi and Aditya Cowsik for detailed comments.
We also thank Hengyuan Hu, Neil Band, Omar Shaikh, Thomas Chen, and Cocolab members for helpful discussions as always.
This work was supported in part by ONR Grant N00014-22-1-2110, NSF Grant 2205084, and the Stanford Institute for Human-Centered Artificial Intelligence (HAI). EBF is a Biohub, San Francisco, Investigator.
\newpage
\bibliographystyle{abbrvnat}
\bibliography{references}
\newpage
\appendix
\section{Appendix}
\begin{figure}[!ht]
\centering
\begin{tikzpicture}[
    scale=0.80, every node/.append style={transform shape},
    every node/.style={font=\sffamily},
    block/.style={draw=border, thick, rounded corners=3pt, minimum height=0.8cm,
                  inner sep=5pt, font=\small\sffamily},
    grad/.style={-{Stealth[length=6pt]}, very thick, gradcol, dashed},
    flow/.style={-{Stealth[length=5pt]}, thick, border!80},
]


\node[block, fill=yellow!20, minimum width=2.8cm] (mask) at (-1.8, 5.8)
    {Replay attention mask};

\node[block, fill=blue!5, minimum width=2.8cm] (tokens) at (2.4, 5.8)
    {Token sequence $o_i$};


\node[block, fill=blue!10, minimum width=3.0cm] (model) at (0.3, 4.4)
    {$\pi_\theta$\; (forward pass)};


\node[block, fill=blue!15, minimum width=2.6cm] (toklogp) at (-2.8, 2.8)
    {$\log \pi_\theta(o_{i,t} \mid \mathcal{C}_{i, t})$};

\node[block, fill=red!10, minimum width=2.6cm] (losstok) at (-2.8, 1.2)
    {$\mathcal{L}_{\text{token}}$};


\node[block, fill=orange!15, minimum width=2.6cm,
      draw=autogradcol, line width=1.5pt] (feat) at (3.4, 2.8)
    {Features $(\mathbf{k},\mathbf{h})_t$};

\node[block, fill=green!20, minimum width=2.6cm] (memlogp) at (3.4, 1.2)
    {$\log \pi_\theta(\boldsymbol{\sigma}_{i,t} \mid \mathcal{H}_{i,t})$};

\node[block, fill=red!10, minimum width=2.6cm] (lossmem) at (3.4, -0.4)
    {$\mathcal{L}_{\text{mem}}$};


\node[block, fill=purple!15, minimum width=2.6cm,
      draw=purple!60, line width=1.5pt] (adv) at (-2.8, -0.4)
    {Advantage $A_i$};

\draw[flow] (mask) |- (model);
\draw[flow] (tokens) |- (model);

\draw[flow] (model.south) -- ++(0, -0.4) -| (toklogp.north);
\draw[flow] (model.south) -- ++(0, -0.4) -| (feat.north);

\draw[flow] (toklogp) -- (losstok);

\draw[flow] (feat) -- (memlogp);
\draw[flow] (memlogp) -- (lossmem);

\draw[flow, purple!70] (adv.north)  -| (losstok.south);
\draw[flow, purple!70] (adv.east) |- (lossmem.west);

\draw[grad] (losstok.west) -- ++(-0.7, 0) |- (toklogp.west);
\draw[grad] ([xshift=-0.7cm]toklogp.west) |- ([yshift=-0.3cm]model.west);
\node[font=\large\bfseries, text=gradcol, anchor=east] 
    at ([xshift=-0.85cm, yshift=0.8cm]toklogp.west) {$\nabla_\theta$};


\draw[grad] (lossmem.east) -- ++(0.8, 0) |- (memlogp.east);
\draw[grad] ([xshift=0.8cm]memlogp.east) |- (feat.east);
\draw[grad] ([xshift=0.8cm, yshift=0.8cm]memlogp.east) |- ([yshift=-0.3cm]model.east);
\node[font=\large\bfseries, text=gradcol, anchor=west] 
    at ([xshift=0.95cm, yshift=2.4cm]memlogp.east) {$\nabla_\theta$};


\node[font=\small\sffamily\bfseries, text=autogradcol, anchor=west, 
      text width=3cm, align=left] (autograd-label)
    at ([xshift=1.8cm, yshift=0.0cm]feat.east) {on autograd\\[-2pt] graph};

\draw[autogradcol, thick, -{Stealth[length=4pt]}] 
    (autograd-label.west) -- (feat.east);

\end{tikzpicture}
\caption{
\textbf{Replay masks route gradients through eviction decisions.}
We perform a forward pass over the tokens with $\pi_{\theta}$ using replay attention masks (see Figure~\ref{replay-mask}).
Both losses share the same group-normalized advantage $A_i$ (purple) but differ in their log-probabilities: $\mathcal{L}_{\text{token}}$ uses the per-token next-token log-probs, while $\mathcal{L}_{\text{mem}}$ uses the Gumbel-top-$k$ log-prob of the eviction decisions.
Crucially, since the inputs that are used to make eviction decisions $(\mathbf{k}, \mathbf{h})_t$ are collected during this forward pass, they are on the autograd graph.
Therefore, gradients from $\mathcal{L}_{\text{mem}}$ flow back into $\theta$ and, as a consequence, eviction decisions are treated as a \emph{native} action of the model that is optimized end-to-end just like tokens in the chain-of-thought.
}
\label{fig:autograd}
\end{figure}
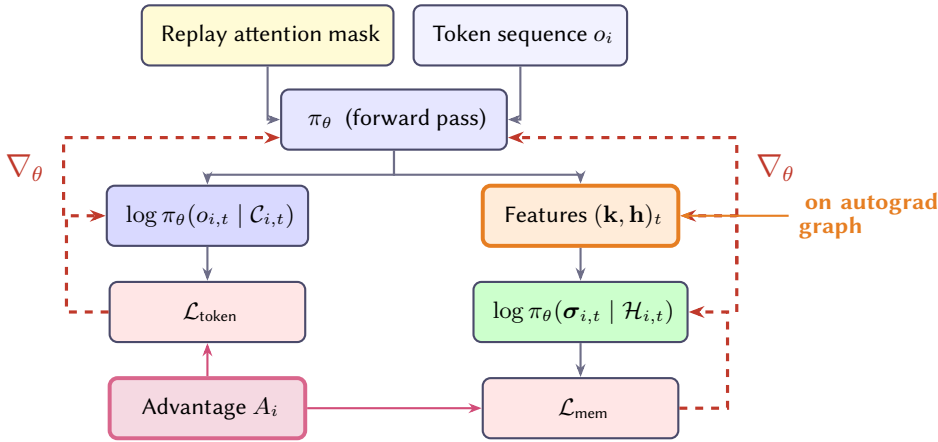
\subsection{Hyperparameters}
\label{sec:hypers}
\begin{table}[ht]
\centering
\caption{Countdown training hyperparameters.}
\label{tab:hyperparams}
\begin{tabular}{lc}
\toprule
\textbf{Parameter} & \textbf{Value} \\
\midrule
Maximum response length & 1024 tokens \\
Sampling temperature & 0.9 \\
Top-$k$ & 50 \\
\midrule
Optimizer & AdamW \\
Adam parameters $(\beta_1, \beta_2)$ & $(0.9, 0.95)$ \\
Adam $\epsilon$ & $1 \times 10^{-15}$ \\
Gradient norm clipping & 1.0 \\
Learning rate scheduler & Constant \\
Learning rate & $5 \times 10^{-6}$ \\
KL penalty coefficient & 0.0 \\\bottomrule
\end{tabular}
\end{table}

\begin{table}[h]
\centering
\caption{DAPO-17k training hyperparameters.}
\label{tab:hyperparams-dapo}
\begin{tabular}{lc}
\toprule
\textbf{Parameter} & \textbf{Value} \\
\midrule
Maximum response length & 1050 tokens \\
Sampling temperature & 1.0 \\
Top-$k$ & disabled \\
\midrule
Optimizer & AdamW \\
Adam parameters $(\beta_1, \beta_2)$ & $(0.9, 0.999)$ \\
Adam $\epsilon$ & $1 \times 10^{-8}$ \\
Weight decay & 0.0 \\
Gradient norm clipping & 1.0 \\
Learning rate scheduler & Constant \\
Learning rate & $5 \times 10^{-6}$ \\
KL penalty coefficient & 0.0 \\
\bottomrule
\end{tabular}
\end{table}

\textbf{Min response length penalty.} 
Completions whose total length $T_i = |\text{prompt}_i| + |\text{completion}_i|$ falls before the first eviction round do not participate in an eviction round and thus carry no signal. 
For the \texttt{DAPO-17k} experiments, we found that it can be helpful
to set $r_i=0$ if the response length $T_i < L$.
This downgrades short rollouts that happened to receive positive reward, preserving the group size $G$ and the natural scale of negative rewards while restoring a clean learning signal for the eviction head.

\subsection{Cache size converges to a constant under grow-then-evict dynamics}
\label{appendix:proof}

\begin{proposition}[Steady-state maximum cache size under periodic eviction]
Consider a cache that undergoes periodic eviction as follows: every $\delta$ newly generated tokens, an eviction round keeps a $(1-\varepsilon)$ fraction of the current cache entries in each layer and permanently removes the remaining $\varepsilon$ fraction, where $\varepsilon \in (0,1]$.
Assume the prefill length is less than $\delta$.

Let $c_t$ denote the cache size in a \emph{single layer} immediately before eviction round $t$.
Then the sequence $\{c_t\}$ converges to the unique fixed point
\[
c^\star = \frac{\delta}{\varepsilon}.
\]
Therefore, since we keep the same fraction in each layer,
the total KV cache size across all $\mathit{L}$ layers immediately before an eviction round converges to
\[
\mathit{C}^{\star} = \mathit{L} \mathit{c^\star} = \mathit{L}\mathit{\frac{\delta}{\varepsilon}}.
\]

\end{proposition}

\begin{proof}
By construction, immediately after eviction round $t$, the cache size in a single layer is $(1-\varepsilon)c_t$.
Before the next eviction round, the model generates $\delta$ new tokens, so we have the following recursion
\[
c_{t+1} = (1-\varepsilon)c_t + \delta.
\]

Since $\varepsilon \in (0,1]$, we have $|1-\varepsilon| < 1$, so this recursion is a contraction and therefore converges to a unique fixed point.
The fixed point condition gives $c_{t+1} = c_t = c^\star$
\[
c^\star = (1-\varepsilon)c^\star + \delta.
\]
Rearranging,
\[
\varepsilon c^\star = \delta,
\qquad\Rightarrow\qquad
c^\star = \frac{\delta}{\varepsilon}.
\]

Since each of the $\mathit{L}$ layers has the same cache size under these dynamics, the total KV cache size immediately before eviction is
\[
\mathit{C}^{\star} = \mathit{L}\mathit{\frac{\delta}{\varepsilon}}.
\]
as desired.
\end{proof}

\subsection{Implementation Details}
\label{appendix:code}
We maintain a custom fork of the Huggingface implementation of Qwen2 as well as a custom fork of the  \texttt{GRPOTrainer} from TRL~\citep{vonwerra2022trl}.
We describe some of the key modifications below.

\paragraph{Per-Layer Attention Masks.}
Because different layers keep different subsets of KV entries after eviction, a single attention mask shared across layers is insufficient, for applying the replay mask mechanism described in Section~\ref{sec:replay}.
The standard Huggingface transformers library does not support this.
We extend the transformer decoder layers to accommodate layer-level attention masks.
We store a list of per-layer attention masks inside a modified \texttt{DynamicCache} object (\texttt{attention\_mask\_list}); at each forward pass, we detect whether we have a layer level mask and, if so, use it in the appropriate decoder layer.

\paragraph{Reconstructing Replay Attention Masks from Retention Decisions.}
To re-run the model under fixed eviction decisions during the replay pass, we must reconstruct the exact (per-layer level) attention masks that would have been in effect at each token position during the original rollout.
Given per-layer retention decisions $\{d_r^{(\ell)} \in \{0,1\}^{B \times |\mathcal{A}_r|}\}_{r=0}^{R-1}$, where $B$ is the batch size, $R = \lceil T/L \rceil - 1$ is the number of eviction rounds, $T$ is the total sequence length, $L$ is the fixed eviction period (denoted $\delta$ in the main text), and $\mathcal{A}_r$ denotes the alive set at round $r$, we construct a dense Boolean mask $M^{(\ell)} \in \{0,1\}^{B \times T \times T}$ independently for each layer $\ell$ as follows.

The sequence is partitioned into $\lceil T/L \rceil$ contiguous blocks, where block $r$ spans global token indices $[rL,\, \min((r+1)L,\, T))$. Within each block, queries attend causally to earlier tokens in the same block via the standard lower-triangular mask. The alive set is initialized as $\mathcal{A}_0 = \text{block}_0$ and updated after each eviction round: applying $d_r^{(\ell)}$ to $\mathcal{A}_r$ yields the retained set
\begin{equation}
    \texttt{kept}_r^{(\ell)} = \bigl\{ i \in \mathcal{A}_r : d_{r,i}^{(\ell)} = 1 \bigr\},
\end{equation}
after which the alive set for the next round is $\mathcal{A}_{r+1} = \texttt{kept}_r^{(\ell)} \cup \text{block}_{r+1}$. All queries in block $r+1$ attend to every token in $\texttt{kept}_r^{(\ell)}$, in addition to causally attending within block $r+1$ itself. After iterating over all $R$ block boundaries, the diagonal of $M^{(\ell)}$ is set to 1 to ensure every token attends to itself. The resulting masks $\{M^{(\ell)}\}_\ell$ are injected into our modified \texttt{DynamicCache}, so that the replay forward pass attends to exactly the same context at each layer as existed during the original rollout.

Note that we need 2D attention masks rather than a broadcastable 1D key-side mask. Standard causal masking is compactly representable because visibility is monotonic in query position — once a key becomes visible, it stays visible. Eviction breaks this monotonicity: a key evicted at round $r$ is visible to queries before round $r$ but invisible to queries after, so visibility is neither a function of key position alone (unlike a padding token) nor monotone in query position.

\subsection{Compression via Meta-Tokens}
\definecolor{ngc@vis}{RGB}{60,130,200}
\definecolor{ngc@evictA}{RGB}{220,80,80}
\definecolor{ngc@sumcol}{RGB}{80,170,100}
\definecolor{ngc@border}{RGB}{80,80,80}
\definecolor{ngc@bgcol}{RGB}{248,248,252}
\definecolor{ngc@gradc}{RGB}{200,120,50}

\begin{figure}
  \centering
\begin{tikzpicture}[
  background rectangle/.style={fill=ngc@bgcol}, show background rectangle,
  tok/.style={
    draw=ngc@border!50, fill=white, rounded corners=4pt,
    minimum width=1.4cm, minimum height=0.85cm,
    font=\normalsize\sffamily, thick
  },
  evicttok/.style={tok, draw=ngc@evictA!60, fill=ngc@evictA!8,
    text=ngc@evictA!70, font=\normalsize\sffamily},
  sumtok/.style={tok, draw=ngc@sumcol!80, fill=ngc@sumcol!25, text=ngc@sumcol!95,
    font=\normalsize\sffamily\bfseries, minimum width=1.55cm, line width=1pt},
  keeptok/.style={tok, draw=ngc@vis!70, fill=ngc@vis!12, text=ngc@vis!85},
  attn/.style={-{Stealth[length=4.5pt]}, semithick, ngc@vis!55},
  sumattn/.style={-{Stealth[length=4.5pt]}, semithick, ngc@sumcol!65},
  garr/.style={-{Stealth[length=6pt]}, thick, ngc@gradc, dashed},
]
\def\sp{1.55}
\def\tokY{0}

\node[keeptok]  (p0) at (0*\sp, \tokY) {$t_0$};
\node[evicttok] (p1) at (1*\sp, \tokY) {\sout{$t_1$}};
\node[keeptok]  (p2) at (2*\sp, \tokY) {$t_2$};
\node[evicttok] (p3) at (3*\sp, \tokY) {\sout{$t_3$}};
\node[sumtok]   (gs) at (4*\sp, \tokY) {$\mathbf{g}$};
\node[keeptok]  (k5) at (5*\sp, \tokY) {$t_5$};
\node[keeptok]  (k6) at (6*\sp, \tokY) {$t_6$};

\draw[sumattn] (k6.north) to[out=135, in=45]  (gs.north);
\draw[attn]    (k6.north) to[out=125, in=55]  (p2.north);
\draw[attn]    (k6.north) to[out=115, in=65]  (p0.north);

\node[font=\small\sffamily\itshape, text=ngc@border!85, align=left, anchor=west]
  at ($(k6.east)+(0.3,0.4)$) {queries after $\mathbf{g}$ attend\\to $\mathbf{g}$ and surviving keys};

\node[font=\small\sffamily, text=ngc@gradc!95,
      draw=ngc@gradc!50, fill=ngc@gradc!10, rounded corners=3pt,
      inner sep=4pt, line width=0.6pt]
  (rew) at (gs |- 0,-2.1) {task reward};
\draw[garr] (rew.north) -- (gs.south);
\node[font=\scriptsize\sffamily, text=ngc@gradc!90, align=center, anchor=west]
  at ($(rew.east)+(0.2,0)$) {$\nabla_{\!\theta}$ flows into\\ gist embedding};

\end{tikzpicture}
\caption{\textbf{NGC gives compression meta-tokens for free.}
Tokens after the summary token $\mathbf{g}$ can only attend to $\mathbf{g}$ and the KV cache entries the model decides to keep.
Since $\mathbf{g}$ is just a token, it receives a per-token policy gradient.
This gives the model an incentive, via the RL objective, to pack useful information into $\mathbf{g}$'s representation.
If all prefix entries before $\mathbf{g}$ are evicted, the induced attention mask is exactly that of a gist token~\citep{mu2024gist}.}
\label{fig:gist}
\end{figure}
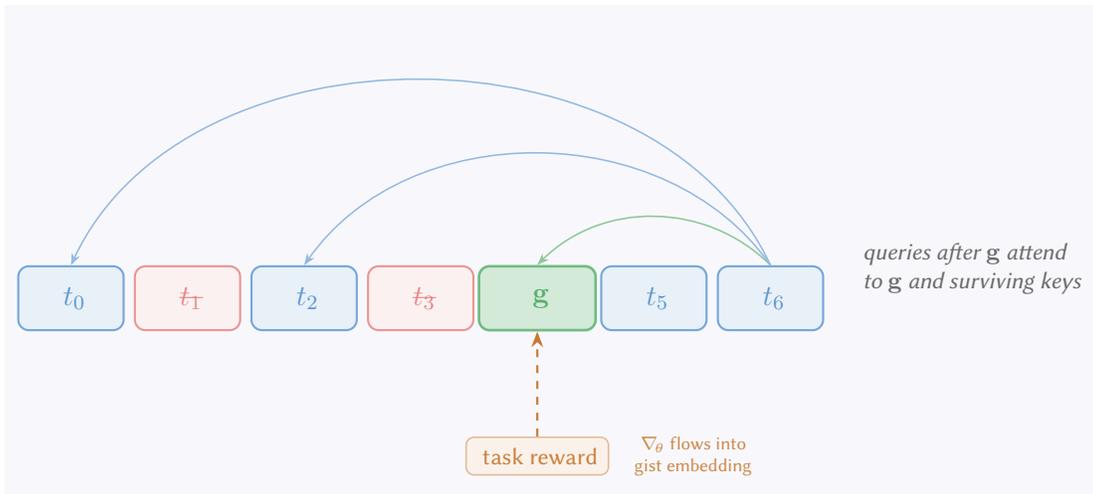
Eviction enables forgetting but does not directly allow for consolidating information.
Consider a model mid-way through a long arithmetic derivation: in addition to dropping earlier steps, it might be helpful to take stock of what has been established before continuing.

We show a simple way to enable such consolidation within the NGC framework, requiring no additional objectives or architecture changes. We illustrate this idea in Figure~\ref{fig:gist}. Immediately before an eviction round, we force the model to emit a special meta-token $\mathbf{g}$ via constrained decoding; its embedding is initialized from the average of the ``tl;dr'' tokens but is otherwise an ordinary entry in the vocabulary.

The subtlety is that $\mathbf{g}$ is emitted deterministically at inference, so its probability is degenerate and admits no sampled log-probability.
During training we score $\mathbf{g}$ as if it were sampled from the model's next-token distribution at that position; this gives it a per-token policy gradient. 
Given this signal, eviction supplies the incentive to use $\mathbf{g}$ strategically. Once cache entries preceding $\mathbf{g}$ are removed, its key--value representation becomes a channel through which information from the evicted entries can still influence subsequent computation. 
The optimization pressure, in principle, teaches the model to route useful information about the prefix through $\mathbf{g}$---yielding gist-like behavior without a separate compression objective.

NGC generalizes the gist token construction in \citet{mu2024gist}: 
gist tokens only compress system prompts and require a distillation 
objective to train, while NGC can exhibit the same behavior for free as a 
consequence of learning to reason and can compress both the prefill and generation tokens.
More precisely, when all entries before $\mathbf{g}$ are evicted, the resulting attention mask for tokens after $\mathbf{g}$ is equivalent to that of the gist attention mask.
This approach follows 
a pattern in which emitting a special token can induce structured 
behavior in an LM~\citep{goyal2024pause,zelikman2024quietstar}.

\paragraph{Gist tokens via LogitsProcessors}
We implement summary tokens using Huggingface's support for logit processors.
At the end of each decode step, the model forward pass checks whether the next position would trigger eviction. 
If so, it sets a boolean flag \texttt{\_force\_summary\_next = True}. On the subsequent call to \texttt{generate()}, the logits processor intercepts the vocabulary logits, fills them with $-\infty$, and assigns a score of $0$ to the summary token id, deterministically forcing its emission. The flag is then cleared.

\end{document}